\documentclass{article}

\PassOptionsToPackage{numbers, compress}{natbib}
\usepackage[preprint]{neurips_2022}
\usepackage{graphicx}
\usepackage{subfigure}
\usepackage{enumitem}
\usepackage{bm}
\usepackage{bbm}
\usepackage{algpseudocode,algorithm,algorithmicx}  

\usepackage{wrapfig}  

\usepackage[title,toc,titletoc,page]{appendix}

\usepackage[utf8]{inputenc} % allow utf-8 input
\usepackage[T1]{fontenc}    % use 8-bit T1 fonts
\usepackage{hyperref}       % hyperlinks
\usepackage{url}            % simple URL typesetting
\usepackage{booktabs}       % professional-quality tables
\usepackage{amsfonts}       % blackboard math symbols
\usepackage{nicefrac}       % compact symbols for 1/2, etc.
\usepackage{microtype}      % microtypography
\usepackage{xcolor}         % colors
\usepackage{amsmath}
\usepackage{amsthm}

\newtheorem{theorem}{Theorem}[section]
\newtheorem{proposition}[theorem]{Proposition}
\newtheorem{lemma}[theorem]{Lemma}

\theoremstyle{definition}

\newtheorem{assumption}[theorem]{Assumption}
\theoremstyle{remark}

\newcommand{\argmin}[1]{\underset{#1}{\operatorname{arg}\,\operatorname{min}}\;}

\let\citet\cite

\title{Quantifying Epistemic Uncertainty in Deep Learning
}

\author{%
  Ziyi Huang, Henry Lam, Haofeng Zhang
  \\
  Columbia University\\
  New York, NY 10027 \\
  \texttt{zh2354,khl2114,hz2553@columbia.edu} \\
}

\begin{document}

\maketitle

\begin{abstract}
Uncertainty quantification is at the core of the reliability and robustness of machine learning. In this paper, we provide a theoretical framework to dissect the uncertainty, especially the \textit{epistemic} component, in deep learning into \textit{procedural variability} (from the training procedure) and \textit{data variability} (from the training data), which is the first such attempt in the literature to our best knowledge. We then propose two approaches to estimate these uncertainties, one based on influence function and one on batching. We demonstrate how our approaches overcome the computational difficulties in applying classical statistical methods. Experimental evaluations on multiple problem settings corroborate our theory and illustrate how our framework and estimation can provide direct guidance on modeling and data collection efforts.
\end{abstract}

\textbf{Update 06/2023:} This work, originally posted on arXiv in 10/2021, is a preliminary draft. Our more recent work \citep{huang2023efficient} is a substantially updated version. \citep{huang2023efficient} inherits and further develops the underlying idea from this work, and moreover, resolves some limitations in this work. For instance, we assume constant procedural bias in Theorem \ref{thm: decomp} and an uncorrelation property of procedural and data   variability in Section \ref{sec: batch}, which appears too restrictive in practice. \citep{huang2023efficient} avoids making such assumptions by establishing a framework based on the NTK theory completely.
    
\section{Introduction}
\vspace{-0.5em} 
Uncertainty quantification is of major importance in machine learning, especially for risk-based decision making and safety-critical applications \citep{michelmore2018evaluating}. It consists of two different components, often referred to as \emph{aleatoric} and \emph{epistemic} uncertainties. Aleatoric uncertainty refers to the intrinsic stochasticity of the problem while epistemic uncertainty refers to the errors coming from the inadequacy of the model or data noises. Measuring epistemic uncertainty is of importance in practice as it assesses the reliability of the predictive model, bearing the following immediate implications:
\begin{enumerate}[leftmargin=*]
\item\emph{Model selection:} Between models with similar predictive power, one with higher epistemic uncertainty is less favorable. That is, information on epistemic uncertainty gives a second dimension to predictive power for model selection. 
% than a model with less epistemic uncertainty. 
% This can be generalized to graphing a tradeoff between prediction accuracy and reliability.
\item\emph{Effort planning: }A quantification of different sources of epistemic uncertainty guides which aspect of the uncertainty needs more attention and improvement. For instance, if more uncertainty appears due to the training noise than the data noise, then one should put more effort into training than collecting more data, and vice versa.
% the lack of , , informative guidance on modeling and data collection can be provided to improve the training process. 
\item\emph{Exploration guidance: }Epistemic uncertainty gives rise to upper confidence bound that has been widely used to tackle the exploration-exploitation dilemma in online learning.
\end{enumerate}
Existing work on uncertainty measurement in deep learning models mainly focuses on \textit{prediction sets} (\textit{predictive uncertainty}), which captures the sum of epistemic and aleatoric uncertainties \citep{pearce2018high,romano2019conformalized,barber2019predictive,alaa2020frequentist,chen2021learning}. Despite some high-level discussions \citep{pearce2018high,hullermeier2021aleatoric}, few studies have been done in understanding epistemic uncertainty at a deeper level, namely differentiating and quantifying its different sources. 

%Beyond predictive uncertainty, quantifying epistemic uncertainty alone can be of independent interest in practice. Epistemic uncertainty assesses the reliability of the predictive model. In addition, epistemic uncertainty (e.g., upper confidence bound) can be used to tackle the exploration-exploitation dilemma in online learning. Moreover, further quantifying different sources of epistemic uncertainty separately can provide direct guidance on modeling and data collection to improve training. %One of our current work aims to quantify epistemic uncertainty in a deep learning model, or more generally, quantify procedural uncertainty and data uncertainty separately.

% In this paper, we consider a theoretical framework to systematically lay out the sources of uncertainty in a general prediction model. We quantify each source of uncertainty statistically, exemplify the implications, and demonstrate how to estimate these uncertainties in deep supervised learning. %Among those, we separately describe the uncertainty in the training procedure and the uncertainty in the training data. 

In this work, we fill in this gap by providing the first systematic study on the epistemic uncertainty in deep supervised learning. Epistemic uncertainty can be subdivided into three main sources: \emph{Model uncertainty} signifies the discrepancy
between the best model in the hypothesis class and the real data distribution, \emph{procedural variability} refers to the noise stemming from the training procedure, and \emph{data variability} arises from the uncertainty over how well the training data represents the real distribution. Our study in particular provides arguably the first dichotomy of procedural and data variabilities in deep learning. As an illustration of the strength of this dichotomy, we apply it to explain the benefit of deep ensemble in reducing procedural rather than data variabilities, and demonstrate how this benefit scales with the ensemble size.
% : A large procedural variability suggests the need to increase ensemble size, while a large data variability suggests the need to collect more data. This guidance can be applied one-off statically, but also in online settings in connection to more efficient exploration.
%to explain how deep ensemble achieves more accurate prediction by \emph{reducing procedural variability}, functioning in a sense similar to the conditional Monte Carlo method.

To execute our framework, we propose two approaches to estimate epistemic variances for well-trained neural networks in practice, which overcome some statistical and computational challenges faced by naive use of classical inference methods such as the bootstrap and jackknife. The first approach uses the delta method that involves computing the influence function, the latter viewed as the (functional) gradient of the predictor with respect to the data distribution. Our novelty lies in the derivation of this function using the theory of neural tangent kernel (NTK) \citep{jacot2018neural} that bypasses the convexity assumption violated by modern neural networks \citep{koh2017understanding,basu2020influence}. The second approach uses the batching method that groups data into a small number of batches upon which re-training applies. This approach bypasses the time-consuming Gram matrix inversion in the influence function calculation, while advantageously expending a minimal re-training effort -- In the case of deep ensemble, this effort is essentially the same as building the ensemble. While batching is long known in simulation analysis \citep{Glynn1990simulation,Schmeiser1982batch}, its application and discovered advantages in deep learning are new to our best knowledge. 

Our main contributions can be summarized as follows:\\
(1) We propose the first theoretical framework to dissect epistemic uncertainty in deep learning and give concrete justification and estimation methods for each component.\\
(2) We develop two approaches with theoretical justification to estimate the epistemic variances for deep neural networks, overcoming the statistical and computational challenges faced by classical statistical methods.\\
(3) We conduct comprehensive experiments to show the correctness of our theory and illustrate the practical use of our approaches on model training.

% , albeit at the expense of a looser estimate
% to overcome the difficulty of applying classical statistical approaches to deep learning. 
% (2) Based on NTK, we propose an influence function method for epistemic variance estimation, which bypasses the convexity assumption that is violated by modern neural networks \citep{koh2017understanding,basu2020influence}. (3) We propose an alternative batching method \citep{glynn2018constructing}, which can bypass the time-consuming Gram matrix inversion in the influence function calculation. %(4) We point out that two methods, classical statistical subsampling methods (e.g. bagging and Jackknife) and deep ensemble, are designed to estimate different types of epistemic uncertainty.

\textbf{Related Work.} Ronald A. Fisher \citep{fisher1930inverse} was probably the first one to formally distinguish and build a connection between aleatory and epistemic uncertainty, according to Frank R. Hampel \citep{hampel2011potential}. In modern machine learning, their distinction had been proposed in \citet{senge2014reliable} and then extended to deep learning \citep{kendall2017uncertainties}. The distinction between procedural and data variabilities in epistemic uncertainty was intuitively described in \citet{pearce2018high}, but no estimation method or rigorous definition was provided. We also remark that our work adopts the frequentist view, which is different from Bayesian uncertainty quantification  \citep{kendall2017uncertainties,gal2016dropout,he2020bayesian}. Further related work will be described in detail in the relevant subsections.

% no estimation method is given. To the best of our knowledge, those two kinds of variability has 
%\textbf{Subsampling Methods.}

%\textbf{Ensemble Methods in Deep Learning.}

%\textbf{Influence Functions and Infinitesimal Jackknife.}

%\textbf{Prediction Intervals.}

 \vspace{-0.5em} 
\section{Statistical Framework of Uncertainty}
 \vspace{-0.5em} 
In this section, we first review the standard statistical learning framework and then provide rigorous definitions for each type of uncertainty. %rigorously define each type of uncertainty. %

In supervised learning, we assume that the input-output pair $(X,Y)$ is a random vector where $X\in  \mathcal{X}\subset \mathbb{R}^d$ is an input and $Y\in  \mathcal{Y}$ is the corresponding output. The random vector follows some unknown probability measure $\pi$ on $\mathcal{X} \times \mathcal{Y}$. A learner is given access to a set of training data $\mathcal{D}_{tr}=\{(x_1,y_1),(x_2,y_2),...,(x_n,y_n)\}$, which is assumed to be independent and identically distributed (i.i.d.) according to the distribution $\pi$. Let the conditional distribution of $Y$ given $X$ be $\pi_{Y|X}(y|x)$. Let $\bm{x}=(x_1,...,x_n)^T$ and $\bm{y}=(y_1,...,y_n)^T$ for short. Let $\mathcal{H}$ denote a hypothesis class $\mathcal{H}=\{h_\theta: \mathcal{X} \to \mathcal{Y} | \theta \in \Theta\}$ where $\theta$ is the parameter (e.g., collection of network weights in deep learning) and $\Theta$ is the set of all possible parameters. Let $\mathcal{L}: \mathcal{Y} \times \mathcal{Y} \to \mathbb{R}$ be the loss function. For instance, $\mathcal{L}$ can be the square error in regression or cross-entropy loss in classification. The goal of the learner is to find a
hypothesis $h_{\theta^*} \in \mathcal{H}$ that minimizes the following (population) risk function:
$R(h):=\mathbb{E}_{(X,Y)\sim \pi} [\mathcal{L}(h(X),Y)].$
Let the true risk minimizer be
\begin{equation} \label{equ:trueminimizer}
h_{\theta^*}:= \argmin {\theta\in \Theta} R(h_\theta).    
\end{equation}
In practice, the population risk function is unknown (since $\pi$ is unknown) so an empirical risk function is derived based on the training data $\mathcal{D}_{tr}$:
$\hat{R}(h):= \mathbb{E}_{(X,Y)\sim \hat{\pi}_{\mathcal{D}_{tr}}} [\mathcal{L}(h(X),Y)]=\frac{1}{n} \sum_{i=1}^{n} \mathcal{L}(h(x_i),y_i)$
where $\hat{\pi}_{\mathcal{D}_{tr}}$ is the empirical distribution on the training data. 
% This is the performance of a hypothesis with respect to the training data. Then the learner will find a good hypothesis $h_\theta$ by minimizing the empirical risk:
The learner aims to obtain
\begin{equation} \label{equ:empiricalminimizer}
h_{\hat{\theta}^*}:= \argmin {\theta\in \Theta} \hat{R}(h_\theta)
\end{equation}
which is called empirical risk minimizer. In general, $h_{\theta^*}$ is not equal to $h_{\hat{\theta}^*}$. Note that in deep learning, $\hat{\theta}^*$ cannot be obtained precisely because of the non-convex nature of neural networks. Therefore, we usually approximate $\hat{\theta}^*$ by $\hat{\theta}$ using gradient-based methods. We remark that $h_{\theta^*}$ is the best choice (in the sense of minimizing the risk) within the hypothesis set $\mathcal{H}$. %, depending on the choice of $\mathcal{H}$. 
If any possible functions can be selected, the best predictions (in the sense of minimizing the risk) are
described by the \textit{Bayes} predictor $f^*$ \citep{hullermeier2021aleatoric}:
\begin{equation} \label{equ:Bayes}
f^*(X):=\argmin {\hat{y}\in\mathcal{Y}} \mathbb{E}_{Y \sim  \pi_{Y|X}} [\mathcal{L}(\hat{y}, Y)|X].
\end{equation}
$f^*$ cannot be obtained in practice, since the conditional distribution $\pi_{Y|X}(y|x)$ is unknown.

 \vspace{-0.5em} 
\subsection{Epistemic Uncertainty}
\vspace{-0.5em} 
\begin{wrapfigure}{r}{0.4\textwidth}
\footnotesize
    \centering
    \includegraphics[width=0.4\textwidth]{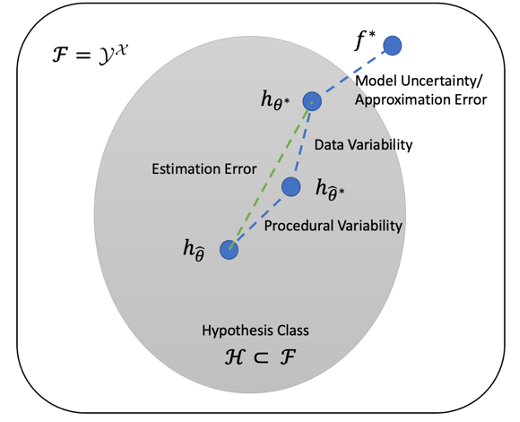}
    \caption{Three sources in epistemic uncertainty.}
    \label{myfigure}
\end{wrapfigure}  

We dissect epistemic uncertainty in deep supervised learning into three sources: model uncertainty, procedural variability, and data variability.

\textbf{Model Uncertainty.} Model uncertainty is the discrepancy between $f^*$ \eqref{equ:Bayes} and $h_{\theta^*}$ \eqref{equ:trueminimizer} which is connected to the inadequacy regarding the choice of the hypothesis class $\mathcal{H}$:
$\text{UQ}_{MU}=h_{\theta^*}-f^*.$
This error is sometimes referred to as the approximation error \citep{mohri2018foundations}. For over-parameterized neural networks where $\mathcal{H}$ is substantially rich, model uncertainty is usually negligible because of the impressive representation power of neural networks. For instance, a sufficiently wide network can universally approximate any continuous functions \citep{cybenko1989approximation,hornik1991approximation,hanin2017approximating} or any Lebesgue-integrable functions \citep{lu2017expressive}.
%Therefore, we treat model uncertainty as $0$ in deep learning. %as previous work; See \cite{pearce2018high}. More precisely, we have
\begin{proposition} \label{thm:modeluncertainty}
Suppose $f^*\in \mathcal{H}$. Then $\text{UQ}_{MU}=0$.
\end{proposition}

\textbf{Procedural Variability.} The empirical risk minimizer in \eqref{equ:empiricalminimizer} could not be obtained in deep learning because of the non-convex nature of neural networks. The standard way to solve \eqref{equ:empiricalminimizer} is to first randomly initialize the network parameters, which introduces the first additional epistemic uncertainty from \emph{random initialization}. Then we do a mini-batch gradient descent on the shuffled training data to finalize the parameter $\hat{\theta}$, which introduces the second additional epistemic uncertainty due to \emph{data ordering}. These lead to uncertainty in $\hat{\theta}$ associated with the training procedure.
% , i.e., uncertainty that is present even with deterministic data or infinite data size.

% , and we call \emph{random initialization} for convenience

We write $\Gamma$ as the random variable corresponding to procedural variability. The trained parameter $\hat{\theta}$ is an instantiation of the outcome from one training run based on $\Gamma$ and $\hat{\pi}_{\mathcal{D}_{tr}}$. %Let $q_{\Gamma}$ be its distribution.
Moreover, we reserve $\Gamma_0$ as the random variable that represents the random initialization of parameters. If the random initialization is the only uncertainty in procedural variability, then $\Gamma$ is a deterministic function of $\Gamma_0$ and they represent the same procedural uncertainty.
%After our training procedure, we use a realization of the training noise $\hat{\theta}$ (along with the empirical distribution $\hat{\pi}_{\mathcal{D}_{tr}}$), to build the model. Thus, we output $h(\hat{\theta}, \hat{\pi}_{\mathcal{D}_{tr}};x_0)$ as our prediction.
%Let $\delta_{\hat{\theta}^*}$ be the point mass distribution at the point $\hat{\theta}^*$ defined in \eqref{equ:empiricalminimizer}. More rigorously, if the optimal solution for \eqref{equ:empiricalminimizer} is not unique, we should think $\delta_{\hat{\theta}^*}$ as the uniform distribution defined on the set of all optimal solutions. 
The uncertainty from procedural variability can be written as
%$$\text{UQ}_{PV}=\delta_{\hat{\theta}^*}-q_{\Gamma}.$$
%Passing this distribution discrepancy to the function discrepancy, we have that 
$\text{UQ}_{PV}=h_{\hat{\theta}}-h_{\hat{\theta}^*}.$

\textbf{Data Variability.} Data variability measures the representative level of the training dataset. This is the most common epistemic uncertainty in classical statistics. The training dataset $\mathcal{D}_{tr}$ is a set of finite samples drawn from $\pi$, which cannot fully represent the population distribution $\pi$ (i.e., $\hat{\pi}_{\mathcal{D}_{tr}}\neq\pi$). Expressed at the hypothesis level, this uncertainty is written as 
$\text{UQ}_{DV}=h_{\hat{\theta}^*}-h_{\theta^*}.$
% In other words, it is the discrepancy between $\pi$ and $\hat{\pi}_{\mathcal{D}_{tr}}$:
% $$\text{UQ}_{DV}=\pi-\hat{\pi}_{\mathcal{D}_{tr}}.$$ this distribution discrepancy  function discrepancy

%  the total epistemic uncertainty
\textbf{Summary.} The three main sources of epistemic uncertainty add up to the total epistemic uncertainty:
$h_{\hat{\theta}}-f^*=:\text{UQ}_{EU}=\text{UQ}_{MU}+\text{UQ}_{PV}+\text{UQ}_{DV}$
which measures the discrepancy between $h_{\hat{\theta}}$ and $f^*$. We remark that $\text{UQ}_{PV}+\text{UQ}_{DV}$ is sometimes referred to as the estimation error in machine learning \citep{mohri2018foundations}. Figure \ref{myfigure} illustrates these three sources in epistemic uncertainty. Lastly,
% \subsection{Other Types of Uncertainty}
we present additional discussions on aleatoric uncertainty and predictive uncertainty in the Supplementary Material for completeness; see Section \ref{sec: other}.

%\newlength{\oldintextsep}
%\setlength{\oldintextsep}{\intextsep}
%\setlength\intextsep{0pt}

 \vspace{-0.5em} 
\section{Quantifying Epistemic Uncertainty}
 \vspace{-0.5em} 
Recall that the goal of supervised learning is to make prediction $h_{\hat{\theta}}$ on an unforeseen test point. Thus, quantifying epistemic uncertainty is important for risk-aware decision making, as a higher variance signifies a higher risk, and also for exploration in online learning as the variance is used in the principle of optimism in the face of uncertainty \citep{auer2002finite,abbasi2011improved,zhou2020neural}. A natural approach to quantify epistemic uncertainty is to estimate the epistemic variance of $h_{\hat{\theta}}(x_0)$ on a given $x_0$. To this end, we have the following decomposition:
\begin{theorem} \label{thm: decomp}
Assume that $\mathbb{E}[\text{UQ}_{PV}|\hat{\pi}_{\mathcal{D}_{tr}}]$ is constant w.r.t. $\hat{\pi}_{\mathcal{D}_{tr}}$ on the fixed data distribution $\pi$. We have
\begin{eqnarray} 
&&\text{Var} (\text{UQ}_{MU}+\text{UQ}_{PV}+\text{UQ}_{DV})\notag\\
&=&\text{Var} (\text{UQ}_{PV}+\text{UQ}_{DV})\label{equ:var1}\\
&=&\text{Var}( \mathbb{E}[\text{UQ}_{DV}|\hat{\pi}_{\mathcal{D}_{tr}}])+\mathbb{E}[\text{Var} (\text{UQ}_{PV}|\hat{\pi}_{\mathcal{D}_{tr}})]\label{equ:var2}
%\approx & \text{Var}( \mathbb{E}[h(\Gamma, \hat{\pi}_{\mathcal{D}_{tr}};x_0)|\hat{\pi}_{\mathcal{D}_{tr}}]) +\text{Var} (h(\Gamma, \pi;x_0))\\
\end{eqnarray}
Moreover, $\text{UQ}_{MU}$, $\text{UQ}_{PV}$ and $\text{UQ}_{DV}$ are uncorrelated.\label{prop:decomposition}
\end{theorem}
In Theorem \ref{prop:decomposition}, we have assumed we plug in the test point $x_0$ into $h$ when computing $\text{UQ}$, and the $\text{Var}$ there is thus taken with respect to both the procedural and data randomness while fixing $x_0$. The assumption on constant bias from procedural variability $\mathbb{E}[\text{UQ}_{PV}|\hat{\pi}_{\mathcal{D}_{tr}}]$ can be justified since, e.g., bias is negligible because of the achievable global minima in neural networks \cite{du2019gradient}. \eqref{equ:var1} follows since $\text{UQ}_{MU}$ is deterministic. \eqref{equ:var2} follows from conditioning and noting that $\mathbb{E}[\text{UQ}_{PV}|\hat{\pi}_{\mathcal{D}_{tr}}]$ is constant and $\text{UQ}_{DV}$ given $\hat\pi_{\mathcal D_{tr}}$ is also constant. By the same argument to obtain \eqref{equ:var2} we also conclude the uncorrelation among the uncertainties, similar to a mixed-effect model in classical statistics.

% can also readily see that all the three uncertainty sources, $\text{UQ}_{MU}$, $\text{UQ}_{PV}$ and $\text{UQ}_{DV}$, are \emph{uncorrelated} 
% We make two remarks:
% \begin{enumerate}
% \item 
% As long as we keep in mind we are evaluating at the test point $x_0$. 
% where a variance term is introduced for exploration, e.g., the uncertainty of the network parameters in the training procedure and After our training procedure, we obtain a realization $\hat{\theta}$ as a result of the procedural variability $\Gamma$ and the empirical distribution $\hat{\pi}_{\mathcal{D}_{tr}}$, to build the model. We output $h_{\hat{\theta}}(x_0)$ as a single model prediction on an unseen test point $x_0$. Note that this quantity is statistics as a function of the data $\mathcal{D}_{tr}$ and the procedural variability $\Gamma$. its epistemic variance  since we are interested in the epistemic uncertainty To quantify this aspect further, our main ingredient is better elaborate the different sources of epistemic uncertainty,  
% Our goal is to estimate
% $\text{Var} (h(\Gamma, \hat{\pi}_{\mathcal{D}_{tr}};x_0))$
% where  
% The following decomposition is elementary:
For convenience in the subsequent discussion, we express $h_{\hat{\theta}}(\cdot)$ as $h(\Gamma, \hat{\pi}_{\mathcal{D}_{tr}};\cdot)$
where $\Gamma$ corresponds to the procedural randomness and  $\hat{\pi}_{\mathcal{D}_{tr}}$ encodes the data noise. We also denote the conditional population mean with respect to the procedural variability as  
$\bar{h}(\hat{\pi}_{\mathcal{D}_{tr}};x_0):=\mathbb{E} [h(\Gamma, \hat{\pi}_{\mathcal{D}_{tr}};x_0)|\hat{\pi}_{\mathcal{D}_{tr}}].$
The above variance decomposition can be equivalently rewritten as
\begin{align} 
&\text{Var} (h(\Gamma, \hat{\pi}_{\mathcal{D}_{tr}};x_0))
= \text{Var}( \bar{h}(\hat{\pi}_{\mathcal{D}_{tr}};x_0))+\mathbb{E}[\text{Var} (h(\Gamma, \hat{\pi}_{\mathcal{D}_{tr}};x_0)|\hat{\pi}_{\mathcal{D}_{tr}})]\label{equ:var}
%\approx & \text{Var}( \mathbb{E}[h(\Gamma, \hat{\pi}_{\mathcal{D}_{tr}};x_0)|\hat{\pi}_{\mathcal{D}_{tr}}]) +\text{Var} (h(\Gamma, \pi;x_0))\\
\end{align}
%where $\pi$ in the last equation refers to the true data distribution. of $h(\Gamma, \hat{\pi}_{\mathcal{D}_{tr}};x_0)$
In either \eqref{equ:var2} or \eqref{equ:var}, the overall epistemic variance 
consists of two terms, where the first term can be interpreted as coming from data variability and the second from procedural variability.
In the following, we illustrate a direct application of our decomposition to explain the statistical power of deep ensemble \citep{lee2015m,lakshminarayanan2017simple,fort2019deep}.

% To estimate the inner conditional expectation and variance, we apply an ensemble technique, which is essentially the deep ensemble method \citep{lee2015m,lakshminarayanan2017simple,fort2019deep}.
 \vspace{-0.5em} 
\subsection{Deep Ensemble Reduces Procedural Variability} \label{sec: DE}
\vspace{-0.5em} 
Compared with other ensemble methods, deep ensemble has empirically demonstrated its state-of-the-art performance on model improvement \citep{ashukha2020pitfalls}. However, few theoretical justification has been developed despite its superior performance. In this section, we argue that the deep ensemble predictor improves the statistical profile over a single model by reducing its procedural variability. The \textit{deep ensemble predictor} $\bar{h}_{m}(x_0)$ at $x_0$ is defined as: 
%In the deep ensemble method, we obtain $h(\hat{\theta}_i, \hat{\pi}_{\mathcal{D}_{tr}};x_0)$ where $i=1,2..,m$ as described above and 
$
\bar{h}_{m'}(x_0):=\frac{1}{m'} \sum_{i=1}^{m'} h_{\hat{\theta}_i}(x_0)
$
where $m'$ is the number of independent training times (with the same training data) and $\hat{\theta}_i$, $i \in [m']$ represent the finalized network parameters.
 %a first application of our UQ framework to understand the benefit of deep ensemble on a rigorous statistical ground. Here, we apply our  to understand the benefit of deep ensemble on a rigorous statistical ground: 
Our theory is formally stated as:
% in the following theorem: %To see this, supposing $m$ is large enough, $\bar{h}_m(\hat{\pi}_{\mathcal{D}_{tr}};x_0)$ gives approaximately
%$$\mathbb{E}[h(\Gamma, \hat{\pi}_{\mathcal{D}_{tr}};x_0)|\hat{\pi}_{\mathcal{D}_{tr}}]$$
%In the deep ensemble method, we train a model (with the same training data) repeatedly each with an independent training run, i.e., with a new random initialization of the network weights. Suppose we repeat the independent training runs $m'$ times each giving a parameter $\hat{\theta}_i$,  where $i \in [m']$.
%In the deep ensemble method, we obtain $h(\hat{\theta}_i, \hat{\pi}_{\mathcal{D}_{tr}};x_0)$ where $i=1,2..,m$ as described above and 
%We then output
%\begin{equation} \label{equ:DE}
%\bar{h}_{m'}(x_0):=\frac{1}{m'} \sum_{i=1}^{m'} %h_{\hat{\theta}_i}(x_0)
%\end{equation}
%as our \textit{deep ensemble predictor} at $x_0$. 
%Deep ensemble has been proven to be the most successful approach to improve prediction accuracy compared with other ensemble methods \citep{ashukha2020pitfalls}, yet few theoretical justification has been given. 
 %a first application of our UQ framework to understand the benefit of deep ensemble on a rigorous statistical ground. Here, we apply our  to understand the benefit of deep ensemble on a rigorous statistical ground: 
%Using our developed uncertainty quantification framework, we argue that the deep ensemble predictor $\bar{h}_{m}(x_0)$ improves the statistical profile over a single model

\begin{theorem} \label{thm:var}
We have
{\small
\begin{align}
&\text{Var} (\bar{h}_{m'}(x_0))%\notag\\
=  \text{Var}(\mathbb{E}[h(\Gamma, \hat{\pi}_{\mathcal{D}_{tr}};x_0)|\hat{\pi}_{\mathcal{D}_{tr}}]) + \frac{1}{m'}\mathbb{E}[ \text{Var}(h(\Gamma, \hat{\pi}_{\mathcal{D}_{tr}};x_0)|\hat{\pi}_{\mathcal{D}_{tr}})]\label{var de}\\
\le & \text{Var}(\mathbb{E}[h(\Gamma, \hat{\pi}_{\mathcal{D}_{tr}};x_0)|\hat{\pi}_{\mathcal{D}_{tr}}]) +\mathbb{E}[ \text{Var}(h(\Gamma, \hat{\pi}_{\mathcal{D}_{tr}};x_0)|\hat{\pi}_{\mathcal{D}_{tr}})]%\notag\\
=  \text{Var}(h_{\hat{\theta}_1}(x_0)).\notag
\end{align}}
Therefore, a deep ensemble predictor has less epistemic variance than a single model predictor. 
\end{theorem}

%From another point of view, deep ensemble is a conditional Monte Carlo estimator of a single model {\color{red} NEED REFERENCES.}, which explains why ensemble improves prediction accuracy.

In \eqref{var de}, by averaging $m'$ training runs with independent procedural randomness, the second term of the variance decomposition of $\bar{h}_{m'}(x_0)$, which corresponds to the procedural variability, is reduced by $m'$ folds compared to $\bar{h}_{m'}(x_0)$. If we could use a very large $m'$, then this procedural variability can be completely removed, which resembles the conditional Monte Carlo method \citep{asmussen2007stochastic} in variance reduction. Nonetheless, as long as $m'>1$, variance reduction for $\bar{h}_{m'}(x_0)$ over $h_{\hat{\theta}_1}(x_0)$ is already present, and indeed a small number of single model predictors are used in the ensemble in practice. We also contrast our explanation on deep ensemble to bagging \citep{breiman1996bagging}. The former averages over training runs while the latter over data resample. Bagging is known to improve prediction accuracy by smoothing base predictors \citep{buhlmann2002analyzing}, and is more effective when these base predictors are non-smooth such as decision trees. This is very different from the procedural variability reduction we have asserted for deep ensemble.

% For the first component in \eqref{equ:var}, we may use the ensemble average to estimate the inner expectation
% %\begin{align} \label{equ:firstcomponent}
% %&\text{Var}(\mathbb{E} [h(\Gamma,\hat{\pi}_{\mathcal{D}_{tr}};x_0)|\hat{\pi}_{\mathcal{D}_{tr}}]) = \text{Var}(\bar{h}(\hat{\pi}_{\mathcal{D}_{tr}};x_0)) \approx  \nonumber\\ 
% %&\text{Var} \left(\frac{1}{m} \sum_{i=1}^m h(\hat{\theta}_j, \hat{\pi}_{\mathcal{D}_{tr}};x_0)\right)=\text{Var} (\bar{h}_m(\hat{\pi}_{\mathcal{D}_{tr}};x_0)).
% %\end{align}
% \begin{align} \label{equ:firstcomponent}
% \text{Var}(\mathbb{E} [h(\Gamma,\hat{\pi}_{\mathcal{D}_{tr}};x_0)|\hat{\pi}_{\mathcal{D}_{tr}}])\approx \text{Var} (\bar{h}_{m'}(x_0)).
% \end{align}
% with a bias vanishing at $O(\frac{1}{m'})$ due to Theorem \ref{thm:var}. However, we cannot directly estimate the outer variance on the data only based on those ensemble models. %since the ensemble models $\bar{h}_m(\hat{\pi}_{\mathcal{D}_{tr}};x_0)$ . %The outer variance on the data needs to be treated more carefully.

\subsection{Challenges of Existing Inference Approaches}
 
In the rest of this paper, we focus on the estimation of epistemic variance, namely the two terms in \eqref{equ:var} or their sum, or \eqref{var de} in the deep ensemble case.

% . For convenience, let
% \begin{align}
% \frac{\sigma^2}{n}&:=\text{Var}(\mathbb{E}[h(\Gamma, \hat{\pi}_{\mathcal{D}_{tr}};x_0)|\hat{\pi}_{\mathcal{D}_{tr}}])\label{DV def}\\
% \tau^2&:=\mathbb{E}[ \text{Var}(h(\Gamma, \hat{\pi}_{\mathcal{D}_{tr}};x_0)|\hat{\pi}_{\mathcal{D}_{tr}})]\label{PV def}
% \end{align}
% be the data and procedural variability components in \eqref{equ:var} respectively (we will see later why the data variance scales with the data size $n$ as $\sigma^2/n$). We would like to estimate $\sigma^2/n$ and $\tau^2$, or simply $\text{Var}(h_{\hat{\theta}_1}(x_0))=\sigma^2/n+\tau^2$ or $\text{Var} (\bar{h}_{m'}(x_0))=\sigma^2/n+\tau^2/m'$ in the deep ensemble case.

Borrowed from classical statistics, two high-level approaches may be considered to estimate these variances. First is an analytical approach using the delta method, which involves computing the influence function that acts as the functional gradient of the predictor with respect to the data distribution. Second is to use resampling, such as the bootstrap or jackknife. Unfortunately, each approach runs into its respective barrier. The performance of the delta method highly depends on the quality of the first-order approximation. However, neural networks are in general ``excessively" non-convex and the $M$-estimation underpinning the delta method breaks down. Resampling requires enough resample replications and thus demands the effort of re-training. Even worse, in estimating the data variability (the first term in \eqref{equ:var} or \eqref{var de}), in fact we would need a ``nested" sampling, i.e., first resample the data many times and for each fixed resample, run and average an inner number of re-training with different procedural randomness in order to compute the involved conditional expectation, which is more demanding than a standard bootstrap or jackknife.

In the next two sections, we present two approaches to bypass each respective challenge. Our first approach applies the theory of NTK to derive an influence function that does not utilize convexity assumptions. Our second approach uses batching that divides data into a small number of batches upon whom re-training applies, which does not rely on many resample replications nor a nested sampling. We note that the NTK approach requires the evaluation of the Gram matrix and its inversion. Batching avoids this computation, and thus is recommended for large dataset. After presenting these approaches individually in Sections \ref{sec:IF} and \ref{sec: batch}, we provide a detailed discussion in Section \ref{sec:combine} on how to combine them judiciously.

 \vspace{-0.5em} 
\section{Influence Function Method via NTK}\label{sec:IF}
 \vspace{-0.5em} 
%In this section, we derive the estimate of the data variability (the first term in \eqref{equ:var}) and procedural variability (the second term in \eqref{equ:var}) separately.
 
\subsection{Estimating Influence Function} \label{sec: IF}
 \vspace{-0.5em} 
%\subsection{Estimating Data Variance via Influence Functions} \label{sec: IF} 
% In this section, we compute the first component in \eqref{equ:var} by influence functions. %\citep{liu2018statistically}.

%Then the KR predictor is expressed as 
%$$h_{KR}(x_0)=\hat{k}(x_0)(K+\lambda I)^{-1}\bm{y}$$

%However, we notice that we need to estimate the variance of $y_i$ given $x_i$ in order to complete the estimation. If $\sigma_i^2$ is a constant, say $\sigma^2$, independent of $i\in [n]$, then it can be estimated via
%$$\sigma^2=\frac{1}{m-1}\sum_{i=1}^m (y_i-h_{KR}(x_i))^2$$
%In the presence of heteroskedasticity, however, $\sigma_i^2$ is hard to analyze. Therefore, we introduce the method of infinitesimal Jackknife.

The theory of NTK implies that the neural network (NN) predictor is essentially a kernel regression (KR) predictor whose kernel is the NTK when the network is sufficiently wide; see detailed discussion in Section \ref{sec: NTK}. Therefore, one direction to estimate the epistemic variance is to analyze the variance of its corresponding KR predictor, via its influence function. The influence function has a long history in statistics \citep{hampel1974influence}, and is used as a nonparametric building block in the delta method for variance estimation \citep{fernholz2012mises}. It was introduced in modern machine learning initially for understanding the effect of a training point on a model's prediction \citep{koh2017understanding,koh2019accuracy} but could be fragile for modern NNs \citep{basu2020influence} partially because the convexity assumption in deep learning is violated. We remark that our work focuses on applying the influence function to compute the epistemic variance. Moreover, our work calculates the influence function for the NN predictor, equivalently the KR predictor using the NTK, which is different from the influence function for network weights \citep{koh2017understanding}. Our approach can bypass the convexity assumption and cast another direction to study neural networks by influence functions.

%It is easy to calculate its variance as follows \citep{liu2018statistically}:
%$$\text{Var}(h_{KR}(x_0)|\bm{x}) = $$ See Section \ref{sec: NTK} for definitions of NTK
We estimate the first component in \eqref{equ:var}. Let $\bm{K}=(K(x_i,x_j))_{i,j=1,...,n} \in \mathbb{R}^{n\times n}$ be the NTK Gram matrix evaluated
on training data. %Denote the minimum eigenvalue of $\bm{K}$, $\lambda_0:= \lambda_{\min}(\bm{K})$. 
For a test point $x_0 \in \mathbb{R}^d$, we let $K(x_0, \bm{x}) \in \mathbb{R}^{n}$ be the kernel value evaluated between the testing point and all training points, i.e., $K(x_0, \bm{x}):= (K(x_0, x_1), K(x_0, x_2), ... , K(x_0, x_n))^T$. We first obtain the inner expectation based on \eqref{equ:krpredictor} and \eqref{equ:NN=KR}:
\begin{equation} \label{equ:KRR}
\bar{h}(\hat{\pi}_{\mathcal{D}_{tr}};x_0)= h_0(x_0) +K(x_0, \bm{x})^T(\bm{K}+\lambda n \bm{I})^{-1} (\bm{y}-h_0(\bm{x}))
\end{equation}
and 
$h_0(x_0)= \mathbb{E}_{\theta_0}[h(\theta_0;x_0)]$
where the expectation is with respect to the random initialization ($h_0$ does not depend on the data). There are two ways to approximate $\bar{h}(\hat{\pi}_{\mathcal{D}_{tr}};x_0)$. It can be approximated by the \textit{ensemble average} $\bar{h}_m(x_0)$ for sufficiently large $m$. Or it can be obtained by computing the NTK $K$ and estimating $h_0(x_0)$ in  \eqref{equ:KRR} without training the network at all. In practice, we notice that both methods give us satisfactory performance \citep{lee2019wide}. %In the following, we write
%$\bar{h}(x_0)=\bar{h}(\hat{\pi}_{\mathcal{D}_{tr}};x_0)$ for short.
%$$\bar{h}_m(x_0)=h_0(x_0)
%+K(x_0, \bm{x})^T (\bm{K}+\lambda n \bm{I})^{-1} (\bm{y}-h_0(\bm{x}))$$
%where $h_0(x_0)= \frac{1}{m}\sum_{i=1}^m h(\theta^{(i)}_0;x_0)$ and $\theta^{(i)}_0$ is the $i$-th randomly initialized parameters in the ensemble. Note that $h_0(x_0)$ is independent of the training data. Our first observation is that
%\begin{proposition} \label{thm:losskr}
%$\bar{h}(\hat{\pi}_{\mathcal{D}_{tr}};\cdot)$ is the solution to the following kernel ridge regression problem
%\begin{equation} \label{equ:losskr2}
%\bar{h}-h_0= \argmin{g\in \bar{\mathcal{H}}} \frac{1}{n}\sum_{i=1}^n (y_i-h_0(x_i)-g(x_i))^2 + \lambda \|g\|^2_\mathcal{\bar{H}}.
%\end{equation}
%\end{proposition}

We recall the definition of the influence function. Let $P$ be a distribution on the domain $\mathcal{X}\times \mathcal{Y}$. Let $z = (z_x, z_y) \in \mathcal{X}\times \mathcal{Y}$. Suppose $P_{\varepsilon,z} = (1 -\varepsilon)P + \varepsilon\delta_z$ where $\delta_z$ denotes the point mass distribution in $z$. Let $T$ be a statistical functional $T : P \to T(P)$. Then the
influence function of $T$ at $P$ at the point $z$ is defined as
$IF(z; T, P) = \lim_{\varepsilon\to 0} \frac{
T(P_{\varepsilon,z}) - T(P)}{\varepsilon}.$
Let $T(P)$ be the predictor $\bar{h}(P;x_0)$ based on the data distribution $P$. Note that $\bar{h}(P;x_0)$ has removed the procedural variability. If $P=\hat{\pi}_{\mathcal{D}_{tr}}$, then $T(\hat{\pi}_{\mathcal{D}_{tr}})$ is $\bar{h}(\hat{\pi}_{\mathcal{D}_{tr}};x_0)$. If $P=\pi$, then $T(\pi)$ is $h_{\theta^*}(x_0)$. Under some mild conditions (e.g., $T$ is $\rho_\infty$-Hadamard differentiable; see \citet{fernholz2012mises}), the central limit theorem holds: $\sqrt{n}(T(\hat{\pi}_{\mathcal{D}_{tr}}) - T(\pi))$ is asymptotically normal with mean 0 and variance
\begin{equation} \label{sigma}
\sigma^2 = \int_{z} IF^2(z; T, \pi)d\pi(z).
\end{equation}
In other words, $T(\hat{\pi}_{\mathcal{D}_{tr}})\approx \mathcal{N}(T(\pi),\sigma^2/n)$.
Therefore, we have approximately
$\text{Var}(\bar{h}(\hat{\pi}_{\mathcal{D}_{tr}};x_0))\approx \frac{\sigma^2}{n}\approx \frac{1}{n^2} \sum_{i=1}^{n} IF^2(z_i; T, \pi)$
where $z_i=(x_i,y_i)$ is from the training dataset. This formula gives us the approximation of the first component in \eqref{equ:var}. Hence our main target is to compute the influence function. In fact, we have
% , as discussed in \eqref{equ:firstcomponent}
\begin{theorem} \label{thm:IF}
The influence function is given by
{\small
\begin{equation} \label{equ:IF}
IF(z; T, \hat{\pi}_{\mathcal{D}_{tr}})= K(x_0, \bm{x})^T (\bm{K}+\lambda n \bm{I})^{-1} M_z(\bm{x})-M_z(x_0)
\end{equation}}
where $M_z(\bm{x}):=(M_z(x_1),...,M_z(x_n))^T$ and $M_z(x)$ is 
{\small
\begin{align*}
M_z(x):=\bar{h}(\hat{\pi}_{\mathcal{D}_{tr}};x)-h_0(x)-\frac{1}{\lambda}(z_{y}-\bar{h}(\hat{\pi}_{\mathcal{D}_{tr}};z_x)) K(z_x,x).  
\end{align*}}
\end{theorem}
\begin{theorem} \label{thm:IF2}
Suppose that $T$ is continuously $\rho_\infty$-Hadamard differentiable at $\pi$ and $\hat{\pi}_{\mathcal{D}_{tr}}$. Suppose that $n(\bar{h}(\hat{\pi}_{\mathcal{D}_{tr}};x_0)-h_{\theta^*}(x_0))^2$ is uniformly integrable. Then we have the strong consistency
\begin{align*}
\lim_{n\to \infty}  &\Bigg|\frac{1}{n}\sum_{i=1}^{n} IF^2((x_i,y_i); T, \hat{\pi}_{\mathcal{D}_{tr}})
-n\ \text{Var} (\bar{h}(\hat{\pi}_{\mathcal{D}_{tr}};x_0))\Bigg| =0, \quad a.s.  
\end{align*}
\end{theorem}
Theorem \ref{thm:IF} extends the influence function derivation of kernel ridge regression \citep{christmann2007consistency,debruyne2008model} to provide $IF(z; T, \hat{\pi}_{\mathcal{D}_{tr}})$ at any test point $x_0$. Note that the asymptotic variance \eqref{sigma} is expressed in terms of $IF(z; T, \pi)$, instead of $IF(z; T, \hat{\pi}_{\mathcal{D}_{tr}})$. Therefore, to fill this gap, Theorem \ref{thm:IF2} proves the convergence of $IF(z; T, \hat{\pi}_{\mathcal{D}_{tr}})$ to $IF(z; T, \pi)$ with the aid of continuous differentiability. To conclude, the estimator of the data variability using the influence function is given by
$
\text{Var} (\bar{h}(\hat{\pi}_{\mathcal{D}_{tr}};x_0))
\approx \frac{1}{n^2} \sum_{i=1}^{n} IF^2((x_i,y_i); T, \hat{\pi}_{\mathcal{D}_{tr}}). %\\  &+\frac{1}{m-1} \sum_{i=1}^m \left(h(\hat{\theta}_j, \hat{\pi}_{\mathcal{D}_{tr}};x_0)-\bar{h}_m(\hat{\pi}_{\mathcal{D}_{tr}};x_0)\right)^2
%\approx & \text{Var}( \mathbb{E}[h(\Gamma, \hat{\pi}_{\mathcal{D}_{tr}};x_0)|\hat{\pi}_{\mathcal{D}_{tr}}]) +\text{Var} (h(\Gamma, \pi;x_0))\\
$

%where $m$ is large to make the estimation accurate; see \eqref{equ:bias}. In addition, we can similarly estimate the variance of a deep ensemble predictor $\bar{h}_{m'}$ where $m'$ is small (in contradiction to the large $m$ in the sample average; see Section \ref{sec: DE}) by using Theorem \ref{thm:var}:
%\begin{align*}
%&\text{Var} (\bar{h}_{m'}(\bm{\hat{\theta}}, \hat{\pi}_{\mathcal{D}_{tr}};x_0))
%\approx \frac{1}{n^2} \sum_{i=1}^{n} IF^2((x_i,y_i); T, P) \\  &+\frac{1}{m'(m-1)} \sum_{i=1}^m \left(h(\hat{\theta}_j, \hat{\pi}_{\mathcal{D}_{tr}};x_0)-\bar{h}_m(\hat{\pi}_{\mathcal{D}_{tr}};x_0)\right)^2.
%\approx & \text{Var}( \mathbb{E}[h(\Gamma, \hat{\pi}_{\mathcal{D}_{tr}};x_0)|\hat{\pi}_{\mathcal{D}_{tr}}]) +\text{Var} (h(\Gamma, \pi;x_0))\\
%\end{align*}

%Finally, we remark that the square loss function could be extended to any twice differentiable convex loss function of $\hat{y}_i-y_i$ by the similar technique \citep{debruyne2008model} but corresponding NTK theory is not well-studied.
  \vspace{-0.5em} 
\subsection{Estimating Procedural Variance}
  \vspace{-0.5em} 
%Denote the population mean with respect to the procedural variability:  
%$$\bar{h}(\hat{\pi}_{\mathcal{D}_{tr}};x_0):=\mathbb{E} [(h(\Gamma, \text{Var} (\hat{\pi}_{\mathcal{D}_{tr}};x_0)|\hat{\pi}_{\mathcal{D}_{tr}}])$$
%This could be approximated by \eqref{equ:DE} if $m'$ is sufficiently large. Note that when we view the deep ensemble predictor as an improvement of a single model predictor, $m'$ could be a number as small as $5$ \citep{lakshminarayanan2017simple}. For variance estimation \eqref{equ:var}, however, $m$ should be large to make the estimation accurate, so we term it \textit{ensemble average/variance} and use $m$ to represent a large number in contradiction to the small $m'$ in the deep ensemble predictor \eqref{equ:DE}. 
%The former is used in for variance estimation \eqref{equ:var} so $m$ should be large. The latter is still one predictor as am improvement of the single network predictor and $m=5$ is usually used in previous work \citep{lakshminarayanan2017simple,pearce2018high}.
%Now we use the ensemble average to estimate the conditional expectation and variance in \eqref{equ:var}. 

For the procedural variance (the second component in \eqref{equ:var}), we estimate with the \textit{ensemble variance}:
\begin{eqnarray}  \label{equ:secondcomponent}
\tau^2&:=&\mathbb{E}[\text{Var} (h(\Gamma, \hat{\pi}_{\mathcal{D}_{tr}};x_0)|\hat{\pi}_{\mathcal{D}_{tr}})]\approx %\\
%&\approx&
\text{Var} (h(\Gamma, \hat{\pi}_{\mathcal{D}_{tr}};x_0)|\hat{\pi}_{\mathcal{D}_{tr}}) %\nonumber 
\\
&\approx&\frac{1}{m-1} \sum_{i=1}^m \left(h_{\hat{\theta}_i}(x_0)-\bar{h}_m(x_0)\right)^2:=\text{EV}.\nonumber
\end{eqnarray}
%Besides, we can roughly treat 
%\begin{align*}
%\mathbb{E}[\text{Var} (h(\Gamma, \hat{\pi}_{\mathcal{D}_{tr}};x_0)|\hat{\pi}_{\mathcal{D}_{tr}})] \approx \text{Var} (h(\Gamma, \hat{\pi}_{\mathcal{D}_{tr}};x_0)|\hat{\pi}_{\mathcal{D}_{tr}})
%\end{align*}
The first approximation above follows since, when viewing $\text{Var} (h(\Gamma, \hat{\pi}_{\mathcal{D}_{tr}};x_0)|\hat{\pi}_{\mathcal{D}_{tr}})$ as a smooth function of $\hat{\pi}_{\mathcal D_{tr}}$, this quantity follows a central limit theorem which gives a fluctuation of order $1/\sqrt n$. This results in an approximation error $O_p(1/\sqrt n)$ which is negligible compared to the magnitude of $\text{Var} (h(\Gamma, \hat{\pi}_{\mathcal{D}_{tr}};x_0)|\hat{\pi}_{\mathcal{D}_{tr}})$ itself. The second approximation holds readily when $m$ is large enough. Otherwise, even for small $m$, we could still use
$$\left[\frac{\text{EV}}{\chi^2_{m-1,1-\alpha/2}/(m-1)},\frac{\text{EV}}{\chi^2_{m-1,\alpha/2}/(m-1)}\right]$$
as a $(1-\alpha)$-level confidence interval (CI) for $\mathbb{E}[\text{Var} (h(\Gamma, \hat{\pi}_{\mathcal{D}_{tr}};x_0)|\hat{\pi}_{\mathcal{D}_{tr}})]$, where $\chi^2_{m-1,\kappa}$ denotes the $\kappa$-th quantile of the $\chi^2$-distribution with degree of freedom $m-1$. This is assuming that the procedural variability is Gaussian and thus we can use the $\chi^2$-distribution for inference on variance. The latter is related to the batching method that we discuss next. %More rigorously, we have
%\begin{theorem} \label{thm:EV}
%Suppose $\mathbb{E}[h(\Gamma, \hat{\pi}_{\mathcal{D}_{tr}};x_0)^2]< \infty$ and the statistical functional $T_2: P\to \text{Var} (h(\Gamma, P;x_0)|P)$ is weak continuous. Then we have the strong consistency:
%$$\lim_{n\to \infty} \lim_{m\to \infty} EV=\tau^2, \ a.s.$$
%\end{theorem}

% is because the bias from the training data is generally small compared with the variance from the training data, especially when we have a large number of training data. %{\color{red} NEED REFERENCES.}

  \vspace{-0.5em} 
\section{Batching} \label{sec: batch}
  \vspace{-0.5em} 
In this section, we propose an alternative approach, based on \textit{batching}  \citep{Glynn1990simulation,Schmeiser1982batch,schruben1983confidence}, to avoid the Gram matrix inversion in the influence function calculation. Originating from simulation analysis, the key idea of batching is to construct a self-normalizing $t$-statistic that ``cancels out" the unknown variance, leading to a valid confidence interval without explicitly needing to compute the variance. It can be used to conduct inference on serially dependent simulation outputs where the standard error is difficult to analytically compute. Previous studies have demonstrated the effectiveness of batching on the use of inference for Markov chain Monte Carlo \citep{geyer1992practical,flegal2010batch,jones2006fixed} and also the so-called input uncertainty problem \citep{glynn2018constructing}. 

Our approach leverages the self-normalizing idea in batching to construct an epistemic variance estimator via the $\chi^2$-statistic, especially in the context of deep ensemble. We divide the data into several batches of equal size, say $\mathcal{D}_1, ... , \mathcal{D}_K$, where $K$ is a constant and $|\mathcal{D}_i|=l=\frac{n}{K}$, $i\in [K]$. With a random initialization $\Gamma_i$, we train the neural network on batch $\mathcal{D}_i$ once, and obtain the output $\hat{\psi}_i(x_0):= h(\Gamma_i, \hat{\pi}_{\mathcal{D}_i};x_0)$ where $i\in [K]$. We show that, perhaps surprisingly, it is easy to obtain the epistemic variance of a deep ensemble predictor by using the $K$ batched predictors.

Recall in Section \ref{sec: IF} we have argued that $T(\hat{\pi}_{\mathcal{D}_{tr}})\approx \mathcal{N}(T(\pi),\sigma^2/n)$. Suppose further that the training noise is also approximate Gaussian, independent of $\hat{\pi}_{\mathcal{D}_{tr}}$ (recall from Theorem \ref{prop:decomposition} that the uncertainty sources are all uncorrelated). In other words, we assume that
\begin{equation}
h(\Gamma,\hat{\pi}_{\mathcal{D}_{tr}};x_0)\stackrel{d}{=} h_{\theta^*}(x_0)+\left(\frac{\sigma}{\sqrt n}Z+\tau W+ B_0\right)(1+o_p(1))\label{normal}
\end{equation}
where $B_0=\mathbb{E}[\text{UQ}_{PV}|\hat{\pi}_{\mathcal{D}_{tr}}]$ is the systematic bias from procedural variability, $Z$, $W$ are two independent $\mathcal N(0,1)$, and the $o_p(1)$ term is as the data size $n$ and potentially a parameter controlling the procedural noise (such as training time) go to $\infty$. To enhance credibility on this assumption, we have empirically verified it with our experiments; see Section \ref{sec:exp}.
Thus we would obtain, for a deep ensemble predictor $\bar{h}_{m'}(x_0)$,
\begin{align}
&\bar{h}_{m'}(x_0)\stackrel{d}{=}h_{\theta^*}(x_0)+\left(\frac{\sigma}{\sqrt n}Z+\frac{\tau}{\sqrt{m'}}W+B_0\right)(1+o_p(1))\stackrel{d}{=}h_{\theta^*}(x_0)+(\nu Y+B_0)(1+o_p(1))\label{normal ensemble}
\end{align}
where $Y\sim \mathcal{N}(0,1)$ and $\nu=\sqrt{\sigma^2/n+\tau^2/m'}$. This matches Theorem \ref{thm:var} which gives
$$\text{Var} (\bar{h}_{m'}(x_0))=\left(\frac{\sigma^2}{n}+\frac{\tau^2}{m'}\right)(1+o(1))$$
%where arguing the last equality needs to assume $\tau^2(\hat P)\stackrel{p}{\to}\tau^2(P)$ as $n\to\infty$ and $\tau^2(\hat P)$ is uniformly integrable. 
where we would have
$$\mathbb{E}[\text{Var} (h(\Gamma, \hat{\pi}_{\mathcal{D}_{tr}};x_0)|\hat{\pi}_{\mathcal{D}_{tr}})]\approx \tau^2, \ \text{Var} (\mathbb{E}[h(\Gamma, \hat{\pi}_{\mathcal{D}_{tr}};x_0)|\hat{\pi}_{\mathcal{D}_{tr}}])\approx \frac{\sigma^2}{n}.$$
Now for the batched predictors $\hat\psi_i(x_0)$, we have similarly
\begin{equation}
\hat\psi_i(x_0)\stackrel{d}{=}h_{\theta^*}(x_0)+\left(\frac{\sigma}{\sqrt l}Z_i+\tau W_i+B_0\right)(1+o_p(1))\label{normal batch}
\end{equation}
where $Z_i,W_i,i\in [K]$ are independent $\mathcal N(0,1)$.  To estimate $\text{Var}(\bar{h}_{m'}(x_0))$, we set $K=m'$ so that each batched predictor satisfies
\begin{equation}
\hat\psi_i(x_0)\stackrel{d}{=}h_{\theta^*}(x_0)+(\sqrt{m'}\nu Y_i+B_0)(1+o_p(1))\label{normal batch equal}
\end{equation}
Note that the RHS in \eqref{normal batch equal} differs from that of \eqref{normal ensemble} only by the $\sqrt{m'}$ factor in front of the $Y$. This turns out to enable the construction of a self-normalizing statistic and in turn an estimate for $\sigma^2/n+\tau^2/m'$, as detailed below:
\begin{theorem}  \label{DE main}
Suppose a single model predictor $h(\Gamma,\hat{\pi}_{\mathcal{D}_{tr}};x_0)$ satisfies
\eqref{normal} and the batched predictors satisfy \eqref{normal batch}.
% $$h(\Gamma,\hat{\pi}_{\mathcal{D}_{tr}};x_0)=h_{\theta^*}(x_0)+\left(\frac{\sigma}{\sqrt n}Z+\tau W\right)(1+o_p(1))$$
% where $Z,W$ are independent $\mathcal{N}(0,1)$ controlling the randomness in the data and training, and $o_p(1)$ is as the data size $n$ and potentially a parameter controlling the training noise go to $\infty$. Analogously, suppose that 
% the batched predictors satisfy
% $$\hat\psi_i(x_0)=h_{\theta^*}(x_0)+\left(\frac{\sigma}{\sqrt l}Z_i+\tau W_i\right)(1+o_p(1))$$
% where $Z_i,W_i,i\in [K]$ are independent $\mathcal{N}(0,1)$. 
% Then, consider the deep ensemble predictors $\bar{h}_{m'}(x_0)$. 
Then, choosing the number of batches $K=m'$ (assumed fixed), and letting $\bar\psi(x_0)=\frac{1}{m'}\sum_{i=1}^{m'}\hat\psi_i(x_0)$ and $S^2=\frac{1}{m'-1}\sum_{i=1}^{m'}(\hat\psi_i(x_0)-\bar\psi(x_0))^2$ be the sample mean and variance of the batched predictors, we have
% \begin{equation}
% \frac{\bar\psi(x_0)-h_{\theta^*}(x_0)}{S/\sqrt {m'}}\Rightarrow t_{m'-1}\label{batching DE}
% \end{equation}
% \begin{equation}
% \frac{\bar{h}_{m'}(x_0)-h_{\theta^*}(x_0)}{S/\sqrt {m'}}\Rightarrow t_{m'-1}\label{sectioning DE}
% \end{equation}
$
\frac{S^2/{m'}}{\sigma^2/n+\tau^2/m'}\Rightarrow\frac{\chi^2_{m'-1}}{m'-1}\label{var est DE}
$
where $\Rightarrow$ denotes convergence in distribution. Therefore,
% \begin{equation}
% \left[\bar\psi(x_0)-t_{m'-1,1-\alpha/2}\frac{S}{\sqrt {m'}},\bar\psi(x_0)+t_{m'-1,1-\alpha/2}\frac{S}{\sqrt {m'}}\right]\label{batching CI DE}
% \end{equation}
% and
% \begin{equation}
% \left[\bar{h}_{m'}(x_0)-t_{m'-1,1-\alpha/2}\frac{S}{\sqrt {m'}},\bar{h}_{m'}(x_0)+t_{m'-1,1-\alpha/2}\frac{S}{\sqrt {m'}}\right]\label{sectioning CI DE}
% \end{equation}
% are asymptotically exact $(1-\alpha)$-level confidence intervals (CIs) for $h_{\theta^*}(x_0)$, where $t_{m'-1,1-\alpha/2}$ is the $(1-\alpha/2)$ quantile of $t_{m'-1}$, and
\begin{equation}
\left[\frac{S^2/m'}{\chi^2_{m'-1,1-\alpha/2}/(m'-1)},\frac{S^2/m'}{\chi^2_{m'-1,\alpha/2}/(m'-1)}\right]\label{var CI DE}
\end{equation}
is an asymptotically exact $(1-\alpha)$-level confidence interval (CI) for $\frac{\sigma^2}{n}+\frac{\tau^2}{m'}$, the asymptotic epistemic variance of the deep ensemble estimator.
\end{theorem}
%In Theorem \ref{DE main}, \eqref{batching DE}, \eqref{sectioning DE} and \eqref{var est DE} follow the same line of proof as Theorem \ref{batching main}. but not statistically consistent in theory (which is typically small) 
Theorem \ref{DE main} shows that with the same number of batched predictors as the training runs used in the deep ensemble, we can construct an accurate-coverage CI for its epistemic variance. Note that by virtue of \eqref{var CI DE} we could also simply use $\frac{S^2}{m'}$ as a rough estimate of $\text{Var}(\bar{h}_{m'}(x_0))$, which is accurate when $m'$ is moderate. We remark that Theorem \ref{DE main} indicates a non-trivial, possibly surprising result: The correct approach to estimate the variance of deep ensemble predictor is to train a \textbf{single, rather than ensemble} neural network predictor in each batch. If we train an ensemble predictor in each batch, we would actually \emph{under-estimate} the uncertainty of the original deep ensemble. Moreover, this result also suggests that batching has an advantageously light computation load, as it only requires the same number of training runs as the original deep ensemble predictor and does not involve any matrix calculation.

\section{Summary and Combining Methods}\label{sec:combine}
%\vspace{-0.5em} 
%For a single model predictor, it is still possible to avoid the Gram matrix inversion by using
In this section, we provide a summary of our proposals on epistemic variance estimation in \eqref{equ:var}:\\
a) IF: $\frac{\sigma^2}{n}$ \eqref{sigma} can be estimated by the influence function, which does not require network training but needs to evaluate the NTK Gram matrix:
$\text{Var} (\bar{h}(\hat{\pi}_{\mathcal{D}_{tr}};x_0))\approx \frac{1}{n^2} \sum_{i=1}^{n} IF^2((x_i,y_i); T, \hat{\pi}_{\mathcal{D}_{tr}}).$\\
b) EV: $\tau^2$ \eqref{equ:secondcomponent} can be estimated by the ensemble variance of $m$ networks:
$\tau^2 \approx \frac{1}{m-1} \sum_{i=1}^m \left(h_{\hat{\theta}_j}(x_0)-\bar{h}_m(x_0)\right)^2.$\\
c) BA: $\frac{\sigma^2}{n}+\frac{\tau^2}{m'}$ can be estimated by the batched predictors which only requires $m'$ times of network training:
$\text{Var}(\bar{h}_{m'}(x_0))\approx \frac{S^2}{m'}.$

We provide some details on how to select estimation methods appropriately. Note that any two estimators from a)-c) give rise to separate estimates of both components in \eqref{equ:var}, so it is sufficient to choose any two of them.
%Any two estimators from a)-c) can give rise to separate estimates of both components in \eqref{equ:var}. 
%A small value of $m$ or $m'$ produces less accurate point estimates of the variances but accurate-coverage CIs. If consistent point estimates of the variances are desired, we can increase $m$ to a large number, rather than $m'$, since doing so would lead to a small batch size $l$ which in turn deteriorates the required Gaussian approximation underpinning the batching method. %Hence, we recommend the use of influence function approach on small datasets as the Gram matrix inversion will not be computationally expensive, while for large datasets, we suggest batching so that the estimation will not be degraded by the dataset dividing.

a)+b): This combination involves matrix calculation in IF and $m$ training runs in building an ensemble in EV. $m$ is determined by the learner: A smaller $m$ can reduce the computational workload but produces less accurate point estimates of the variances. This combination is recommended when the Gram matrix inversion is not computationally expensive, either on small datasets, or on large datasets but with a strategy to reduce the computational cost of the Gram matrix inversion \cite{zhang2015divide}. 

a)+c): This combination involves matrix calculation in IF and $m'$ batched training in BA. $m'$ cannot be too large or too small compared with the size of the dataset: If we increase $m'$ to a large number, then in each batch, we will have a small batch size $l$, which in turn deteriorates the required Gaussian approximation underpinning the batching method. On the other hand, if we decrease $m'$ to a small number, BA produces less accurate point estimates of the variances (but still accurate-coverage CIs). This combination is recommended when the Gram matrix inversion is not computationally expensive, either on small datasets with small $m'$, or on large datasets but with a strategy to reduce the computational cost of the Gram matrix inversion \cite{zhang2015divide}. 

b)+c): This combination involves $m$ training runs in building an ensemble in EV and $m'$ batched training in BA. This combination can avoid matrix inversion calculation and is recommended for large datasets where the BA estimation will not be degraded by the dataset dividing.

\section{Experiments} \label{sec:exp}
%\vspace{-0.5em} 
% The ground truth variance of real-world data is unable to obtain due to the finite number of samples. Therefore,  to quantitatively evaluate our approaches
In this section, we conduct evaluations to demonstrate the effectiveness of our approaches: ensemble variance, influence function, and batching on the epistemic variance estimation \eqref{equ:var}. Since this is the first study that aims to quantify different sources of epistemic uncertainty, we do not find any off-the-shelf algorithms for baseline comparisons. Thus, to fully evaluate our model, we conduct comprehensive experiments on both synthetic datasets and benchmark real-world datasets with different problem settings. We will make our code publicly available.
  %Moreover, to the best of knowledge, our proposals are the first approach that is able to quantify procedural and data variability separately so there are not any off-the-shelf baseline algorithms.

Our approach does not require \textit{any} prior knowledge of the underlying generative process and this knowledge is only used for ground-truth calculation on synthetic datasets. In each experiment on synthetic datasets, the ground-truth variance is empirically calculated with $J = 100$, which is a sufficiently large number of repeated trials leading to stable estimations. In each repetition $j\in [J]$, we generate a new training dataset from the same synthetic distribution and re-train the single predictor/deep ensemble predictor to obtain the prediction $h_j(x_0)$. Then, the ground-truth variance is calculated by the following equation: $\frac{1}{J-1}\sum_{j=1}^J \left(h_j(x_0) - \frac{1}{J}\sum_{i=1}^J h_i(x_0)\right)^2$. %(THE FORMULAS FOR ESTIMATING THE GROUND TRUTHS ARE ACTUALLY DIFFERENT FROM WHAT I THOUGHT, MAYBE DOUBLE CHEKC THEY ARE CORRECT? ALSO, WITH ONLY 100 EXPERIMENTAL REPETITIONS, WOULDN'T THE ESTIMATED GROUND TRUTH BE NOISY?) 
By doing this, we can obtain the ground-truth data variance and procedural variance separately; See Section \ref{sec:expmore1}.

\subsection{Synthetic Datasets} \label{sec:synthetic}

%We conduct experiments on the following synthetic distribution: $X \sim \text{Unif}([0,1]^d)$ and  $Y \sim \sum_{i=1}^d \sin(X^{(i)}) + \mathcal{N}(0,(0.1d)^2)$. The training set $\mathcal{D}_{tr}=\{(x_i,y_i): i=1,...,n\}$ are formed by i.i.d. samples of $(X,Y)$ with sample size $n=200$. We use $x_0=(0.5,...,0.5)$ as the fixed test point and set $K=m'=5$, $m=50$. We implement a two-layer fully-connected NN with $1024$ hidden neurons as the baseline predictor based on the NTK specification in Section \ref{sec: NTK} and it is trained using the regularized square loss \eqref{equ:lossnn}.
\begingroup
\tabcolsep = 0.3pt
\begin{table}[!ht] 
\scriptsize
  \centering
  \begin{tabular}{c|ccc|ccc|ccc}
   \toprule
       & \multicolumn{3}{|c|}{$n=200$} & \multicolumn{3}{|c|}{$n=500$} & \multicolumn{3}{|c}{$n=1000$}\\
      & $\tau^2$ GT &  $\tau^2$ EV & $\tau^2$ Diff & $\tau^2$ GT &  $\tau^2$ EV & $\tau^2$ Diff & $\tau^2$ GT &  $\tau^2$ EV & $\tau^2$ Diff\\
            \midrule
      $d=2$ & $0.5\cdot 10^{-4}$ &  $0.4\cdot 10^{-4}$ &  $-0.1\cdot 10^{-4}$ & $0.4*10^{-4}$ &  $0.3*10^{-4}$ &  $-0.1*10^{-4}$ & $0.4*10^{-4}$ &  $0.3*10^{-4}$ &  $-0.1*10^{-4}$\\
      $d=4$ & $1.6\cdot 10^{-4}$ &  $1.5\cdot 10^{-4}$ &  $-0.1\cdot 10^{-4}$ & $1.6*10^{-4}$ &  $1.5*10^{-4}$ &  $-0.1*10^{-4}$ & $1.5*10^{-4}$ &  $1.4*10^{-4}$ &  $-0.1*10^{-4}$\\
      $d=8$ & $9.3\cdot 10^{-4}$ &  $9.1\cdot 10^{-4}$ &  $-0.2\cdot 10^{-4}$ & $9.2*10^{-4}$ &  $9.0*10^{-4}$ &  $-0.2*10^{-4}$ & $9.2*10^{-4}$ &  $8.9*10^{-4}$ &  $-0.3*10^{-4}$\\
      $d=16$ & $2.1\cdot 10^{-3}$ &  $2.0\cdot 10^{-3}$ &  $-0.1\cdot 10^{-3}$ & $2.1*10^{-3}$ &  $2.0*10^{-3}$ &  $-0.1*10^{-3}$ & $2.0*10^{-3}$ &  $1.9*10^{-3}$ &  $-0.1*10^{-3}$\\
      $d=32$ & $3.7\cdot 10^{-3}$ &  $3.5\cdot 10^{-3}$ &  $-0.2\cdot 10^{-3}$ & $3.7*10^{-3}$ &  $3.6*10^{-3}$ &  $-0.1*10^{-3}$ & $3.6*10^{-3}$ &  $3.3*10^{-3}$ &  $-0.3*10^{-3}$\\
      \midrule
      &  $\frac{\sigma^2}{n}$ GT & $\frac{\sigma^2}{n}$ IF & $\frac{\sigma^2}{n}$ Diff &  $\frac{\sigma^2}{n}$ GT & $\frac{\sigma^2}{n}$ IF & $\frac{\sigma^2}{n}$ Diff &  $\frac{\sigma^2}{n}$ GT & $\frac{\sigma^2}{n}$ IF & $\frac{\sigma^2}{n}$ Diff \\ 
            \midrule
      $d=2$ & $1.6\cdot 10^{-4}$ & $1.8\cdot 10^{-4}$ & $+0.2\cdot 10^{-4}$ & $0.7*10^{-4}$ & $0.9*10^{-4}$ & $+0.2*10^{-4}$ & $0.4*10^{-4}$ & $0.6*10^{-4}$ & $+0.2*10^{-4}$\\
      $d=4$ & $3.3\cdot 10^{-4}$ & $3.5\cdot 10^{-4}$ & $+0.2\cdot 10^{-4}$ & $1.4*10^{-4}$ & $1.8*10^{-4}$ & $+0.4*10^{-4}$ & $0.8*10^{-4}$ & $1.0*10^{-4}$ & $+0.2*10^{-4}$\\
      $d=8$ & $4.7\cdot 10^{-4}$ & $4.8\cdot 10^{-4}$  & $+0.1\cdot 10^{-4}$ & $2.1*10^{-4}$ & $2.4*10^{-4}$  & $+0.3*10^{-4}$ & $1.3*10^{-4}$ & $1.6*10^{-4}$  & $+0.3*10^{-4}$\\
      $d=16$ & $0.7\cdot 10^{-3}$ & $0.8\cdot 10^{-3}$  & $+0.1\cdot 10^{-3}$ & $0.3*10^{-3}$ & $0.4*10^{-3}$  & $+0.1*10^{-3}$ & $0.2*10^{-3}$ & $0.3*10^{-3}$  & $+0.1*10^{-3}$\\
      $d=32$ & $1.5\cdot 10^{-3}$ & $1.7\cdot 10^{-3}$  & $+0.2\cdot 10^{-3}$ & $0.7*10^{-3}$ & $0.8*10^{-3}$  & $+0.1*10^{-3}$ & $0.4*10^{-3}$ & $0.4*10^{-3}$  & $+0.0*10^{-3}$\\
      \midrule
      & $\frac{\sigma^2}{n}+\frac{\tau^2}{m'}$ GT & $\frac{\sigma^2}{n}+\frac{\tau^2}{m'}$ BA & $\frac{\sigma^2}{n}+\frac{\tau^2}{m'}$ Diff & $\frac{\sigma^2}{n}+\frac{\tau^2}{m'}$ GT & $\frac{\sigma^2}{n}+\frac{\tau^2}{m'}$ BA & $\frac{\sigma^2}{n}+\frac{\tau^2}{m'}$ Diff & $\frac{\sigma^2}{n}+\frac{\tau^2}{m'}$ GT & $\frac{\sigma^2}{n}+\frac{\tau^2}{m'}$ BA & $\frac{\sigma^2}{n}+\frac{\tau^2}{m'}$ Diff\\
            \midrule
      $d=2$ & $1.7\cdot 10^{-4}$ & $1.7\cdot 10^{-4}$ & $+0.0\cdot 10^{-4}$ & $0.8*10^{-4}$ & $1.0*10^{-4}$ & $+0.2*10^{-4}$ & $0.5*10^{-4}$ & $0.5*10^{-4}$ & $+0.0*10^{-4}$\\
      $d=4$ & $3.6\cdot 10^{-4}$ & $3.8\cdot 10^{-4}$ & $+0.2\cdot 10^{-4}$ & $1.7*10^{-4}$ & $1.9*10^{-4}$ & $+0.2*10^{-4}$ & $1.1*10^{-4}$ & $1.4*10^{-4}$ & $+0.3*10^{-4}$\\
      $d=8$ & $6.6\cdot 10^{-4}$ & $6.8\cdot 10^{-4}$  & $+0.2\cdot 10^{-4}$ & $3.9*10^{-4}$ & $4.1*10^{-4}$  & $+0.2*10^{-4}$ & $3.1*10^{-4}$ & $3.3*10^{-4}$  & $+0.2*10^{-4}$\\
      $d=16$ & $1.1\cdot 10^{-3}$ & $1.2\cdot 10^{-3}$  & $+0.1\cdot 10^{-3}$ & $0.7*10^{-3}$ & $0.8*10^{-3}$  & $+0.1*10^{-3}$ & $0.6*10^{-3}$ & $0.7*10^{-3}$  & $+0.1*10^{-3}$\\
      $d=32$ & $2.2\cdot 10^{-3}$ & $2.3\cdot 10^{-3}$  & $+0.1\cdot 10^{-3}$ & $1.4*10^{-3}$ & $1.6*10^{-3}$  & $+0.2*10^{-3}$ & $1.1*10^{-3}$ & $1.2*10^{-3}$  & $+0.1*10^{-3}$\\
    \bottomrule
  \end{tabular}
  \caption{Epistemic Variance Estimation on Synthetic Datasets with $n=200, 500, 1000$ with Different Dimensions.}
 \label{results_all}  
\end{table}

We conduct experiments on the following synthetic distribution: $X \sim \text{Unif}([0,0.2]^d)$ and 
$Y \sim \sum_{i=1}^d \sin(X^{(i)}) + \mathcal{N}(0,0.1^2)$ with multiple dimension settings $d=2,4,8,16,32$ to study the dimensionality effect on the epistemic uncertainty. The training set $\mathcal{D}_{tr}=\{(x_i,y_i): i=1,...,n\}$ is formed by i.i.d. samples of $(X,Y)$ with sample size $n$ taking values in $200, 500, 1000$. We use $x_0=(0.1,...,0.1)$ as the fixed test point and set $K=m'=5$, $m=50$. We implement a two-layer fully-connected NN with $1024$ hidden neurons as the baseline predictor according to the NTK specification in Section \ref{sec: NTK} and it is optimized by the regularized square loss \eqref{equ:lossnn}. 

Tables \ref{results_all} %\ref{results4}, \ref{results5}, \ref{results6} 
report the ground-truth (GT) variance and estimates from our proposed ensemble variance (EV), influence function (IF), and batching (BA). As a primary observation, the estimated variances from our methods appear highly consistent with the ground truth on all experiments with different settings, which evidently demonstrates the robustness and effectiveness of our proposals. In addition, we observe that:\\  
1) The procedural variance seems almost uncorrelated with the size of training data, while the data variance reduces roughly at the rate $n^{-1}$ when the size of training data increases. This observation is consistent with our \eqref{normal}.\\
2) Both procedural variance and data variance increase with the dimension of input data. In particular, the procedural variance increases much faster than the data variance. This is probably caused by the increasing number of minima in the networks along with the increasing number of parameters. This suggests that for large dataset, deep ensemble should be applied to reduce the procedural variance to boost performance.\\

\vspace{-1.0em}
\subsection{Real-World Datasets} \label{sec:realworld}
Finally, we conduct a series of experiments on the following widely-used benchmark real-world datasets: Boston, Concrete, Energy, Wine, and Yacht \citep{hernandez2015probabilistic,gal2016dropout,lakshminarayanan2017simple,rosenfeld2018discriminative,pearce2018high,zhu2019hdi}. The test point is chosen to be the mean feature vector. It is impossible to obtain the ground truth of epistemic variance on real-world datasets and hence, in this section, we focus on the implications of our estimates. 

\newlength{\oldintextsep}
\setlength{\oldintextsep}{\intextsep}
\setlength\intextsep{20pt}
\begin{wraptable}{r}{0pt}
\scriptsize
\centering  
  \begin{tabular}{c|c|c|c}
    \toprule
      & $\tau^2$ EV &  $\frac{\sigma^2}{n}$ IF & $\frac{\sigma^2}{n}+\frac{\tau^2}{m'}$ BA\\
            \midrule
      Boston & $9.1*10^{-3}$ &  $0.2*10^{-3}$ &  $2.3*10^{-3}$\\
      Concrete & $1.3*10^{-3}$ &  $0.4*10^{-3}$ &  $0.7*10^{-3}$\\
      Energy & $10.8*10^{-3}$ &  $0.2*10^{-3}$ &  $2.5*10^{-3}$\\
      Wine & $2.0*10^{-3}$ &  $0.1*10^{-3}$ &  $0.6*10^{-3}$\\
      Yacht & $1.3*10^{-3}$ &  $0.1*10^{-3}$ &  $0.5*10^{-3}$\\
    \bottomrule
  \end{tabular}
  \caption{Epistemic Variance Estimation on Real-World Datasets.\\}
 \label{results7}  
\end{wraptable} 

Table \ref{results7} presents the results of our proposed approaches on the benchmark datasets. As shown, on real-world datasets, the procedural variance is the predominant component in epistemic uncertainty. Compared with the synthetic datasets in Section \ref{sec:synthetic} where we add additional data noise on the labels, these benchmark datasets seem to have less noise on the ground-truth labels, and thus have less data uncertainty. Hence, techniques which efficiently reduce the procedural variance (such as deep ensemble), can bring considerable improvements on the model performance. This also provides a theoretical justification of the findings that deep ensemble performs similarly as combining with Bagging methods in previous work \citep{lakshminarayanan2017simple}.
\vspace{-0.5em} 
\section{Concluding Remarks}
\vspace{-0.5em} 
In this paper, we propose the first theoretical framework to dissect epistemic uncertainty in deep learning, with concrete justification and estimation methods for each component. By quantifying different constituents of epistemic uncertainty, our approaches can provide guidance on modeling and data collection to improve training. The evaluations on synthetic datasets fully corroborate our theory. Moreover, from the experiments in real-world datasets, we observe that the procedural variability is most likely to be the predominant component, which suggests the use of ensemble methods to boost model performance. Finally, quantifying epistemic uncertainty can also be very useful for real-world applications such as online learning problems, which is left as an interesting direction of future work.

%\begin{ack}
%Do {\bf not} include this section in the anonymized submission, only in the final paper. You can use the \texttt{ack} environment provided in the style file to autmoatically hide this section in the anonymized submission.
%\end{ack}

\medskip

\bibliographystyle{abbrv}
\bibliography{example_paper}

%%%%%%%%%%%%%%%%%%%%%%%%%%%%%%%%%%%%%%%%%%%%%%%%%%%%%%%%%%%%

%%%%%%%%%%%%%%%%%%%%%%%%%%%%%%%%%%%%%%%%%%%%%%%%%%%%%%%%%%%%

\newpage

\begin{appendices}

We provide further results and discussions in this supplementary material. Section \ref{sec: other} discusses aleatoric uncertainty and predictive uncertainty. Section \ref{sec: NTK} discusses recent NTK work regrading the equivalence between the NN predictor and KR predictor. Section \ref{sec: proofs} presents the missing proofs for results in the main paper. Section \ref{sec:expmore} presents experimental details and more experimental results. %Section \ref{secE} provides detailed guidance on approach selection.

\section{Other Types of Uncertainty} \label{sec: other}

We present an additional discussion on other types of uncertainty for the sake of completeness. Section \ref{sec:aleatoric} presents aleatoric uncertainty and Section \ref{sec:predictive} presents predictive uncertainty.

\subsection{Aleatoric Uncertainty} \label{sec:aleatoric}
We note that if we could remove all the epistemic uncertainty, i.e., letting $\text{UQ}_{EU}=0$, the best predictor we could get is the Bayes predictor $f^*$. However, $f^*$, as a point estimator, is not able to capture the randomness in $\pi_{Y|X}$ if it is not a point mass distribution.

The easiest way to think of aleatoric uncertainty is that it is captured by $\pi_{Y|X}$. At the level of the realized response or label value, this uncertainty is represented by
$$\text{UQ}_{AU}=y-f^*(x)$$

If the connection between $X$ and $Y$ is non-deterministic, the description
of a new prediction problem involves a conditional probability distribution $\pi_{Y|X}$. Standard neural network predictors can only provide a single output $y$. Thus, even
given full information of the distribution $\pi$, the uncertainty in the prediction of a single output $y$ remains. This uncertainty cannot be removed by better modeling or more data.

There are multiple work aiming to estimate the aleatoric uncertainty, i.e., to learn $\pi_{Y|X}$. For instance, conditional quantile regression aims to learn each quantile of the distribution $\pi_{Y|X}$ \citep{koenker2001quantile, meinshausen2006quantile,steinwart2011estimating}. Conditional density estimation aims to approximately describe the density of $\pi_{Y|X}$ \citep{holmes2007fast,dutordoir2018gaussian,izbicki2016nonparametric,dalmasso2020conditional,freeman2017unified,izbicki2017}. However, we remark that these approaches also face their own epistemic uncertainty in the estimation. See \cite{christmann2007svms,steinwart2011estimating,bai2021understanding} for recent studies on the epistemic bias in quantile regression.

\subsection{Predictive Uncertainty} \label{sec:predictive}

In certain scenarios, it is not necessary to estimate each uncertainty separately. A distinction between aleatoric and epistemic uncertainties might appear less significant. The user may only concern about the overall uncertainty related to the prediction, which is called predictive uncertainty and can be thought as the summation of the epistemic and aleatoric uncertainties. 

The most common way to quantify predictive uncertainty is the \textit{prediction set}: We aim to find a map $\hat{C}: \mathcal{X} \to 2^\mathcal{Y}$ which maps an input to a subset of the output space so that for each test point $x_0\in \mathcal{X}$, the prediction set $\hat{C}(x_0)$ is likely to cover the true outcome $y_0$ \citep{vovk2005algorithmic,barber2019limits,barber2019predictive,lei2014distribution,lei2015conformal,lei2018distribution}. This prediction set is also called a prediction interval in the case of regression \citep{khosravi2010lower,khosravi2011comprehensive}. The prediction set communicates predictive uncertainty via a statistical guarantee on the marginal coverage, i.e.,
$$\mathbb{P}(y_0 \in \hat{C}(x_0))\ge 1- \delta$$
for a small threshold $\delta>0$ where the probability is taken with respect to both training data $\mathcal{D}_{tr}$ (epistemic uncertainty) for learning $\hat{C}$ and the test data $(x_0,y_0)$ (aleatoric uncertainty). It is more tempting to obtain a statistical guarantee with only aleatoric uncertainty by considering the probablity conditional on the prediction set, i.e.,
\begin{equation} \label{equ:coverage}
\mathbb{P}(y_0 \in \hat{C}(x_0)|\hat{C})\ge 1- \delta.    
\end{equation}
However, this guarantee is in general very hard to achieve in the finite-sample case. Even asymptotically, \eqref{equ:coverage} is not easy to achieve unless we have a very simple structure on the data distribution \citep{rosenfeld2018discriminative,zhang2019random}. A recent study show that \eqref{equ:coverage} could hold in the finite-sample sense if we could leverage a set of validation data \citep{chen2021learning}.

%\section{Theory of Neural Tangent Kernel} \label{sec: NTKmore}

\section{Theory of Neural Tangent Kernel} \label{sec: NTK}

In this section, we provide a brief review on the theory of neural tangent kernel and discuss one of its results employed by our work.

NTK \citep{jacot2018neural} has attracted a lot of attention recently since it provides a new perspective on training dynamics, generalization, and expressibility of over-parameterized neural networks. Recent papers show that when training an over-parametrized neural network, the weight matrix at each layer is close to its initialization \citep{li2018learning,du2018gradient}. An over-parameterized neural network can rapidly reduce training error to zero via gradient descent, and thus finds a global minimum despite the objective function being non-convex \citep{du2019gradient,allen2019learning,allen2019convergence,allen2019convergence2,zou2020gradient}. Moreover,
the trained network also exhibits good generalization property \citep{cao2019generalization,cao2020generalization}.
%Moreover, under some conditions, the trained net also exhibits good generalization [Du et al., 2019, 2018b, Li and Liang, 2018,Allen-Zhu et al., 2018a,b, Zou et al., 2018, Arora et al., 2019, Cao and Gu, 2019]. 
These observations are implicitly described by a notion, NTK, suggested by \citet{jacot2018neural}. This kernel is able to characterize the training behavior of sufficiently wide fully-connected neural networks and build a new connection between neural networks and kernel methods. Another line of work is to extend NTK for fully-connected neural networks to CNTK for convolutional neural networks \citep{arora2019exact,yang2019scaling,li2019enhanced}. NTK is a fruitful and rapidly growing area. 

We focus on the work that is relevant to our epistemic uncertainty quantification: Since the gradient of over-parameterized neural networks is nearly constant and close to its initialization, training networks is similar to kernel regression where the feature map of the kernel is the gradient of networks \citep{lee2019wide}. This observation helps us to establish the equivalence between the NN predictor and the KR predictor. We provide more details about this point.

We consider a fully-connected
neural network defined formally as follows. Denote $g^{(0)}(x) = x$ and $d_0 = d$. Let
$$f^{(l)}(x) = W^{(l)} g^{(l-1)}(x), \ g^{(l)}(x) = \sqrt{\frac{2}{d_{l}}} \sigma(f^{(l)}(x))$$
where $W^{(l)} \in \mathbb{R}^{d_l \times d_{l-1}}$ is the weight matrix in the $l$-th layer, $\sigma$ is the coordinate-wise ReLu activation and $l\in [L]$. The output of the neural network is
$$h_\theta(x) = f^{(L+1)}(x) = W^{(L+1)} g^{(L)}(x)$$
where $W^{(L+1)} \in \mathbb{R}^{1 \times d_L}$, and $\theta = (W^{(1)}, ... ,W^{(L+1)})$
represents all the parameters in the network. Suppose the dimension of $\theta$ is $p$. We randomly initialize all the weights to be i.i.d. $\mathcal{N}(0, 1)$ random variables. In other words, we let $\Gamma_0 \sim \mathcal{N}(0, \bm{I}_p)$ and let $\theta_0=\theta(0)$ be an instantiation of $\Gamma_0$.
This initialization method is essentially the He initialization \citep{he2015delving}. %Let $\theta_0$ be the initialization of the network parameters (which is a Gaussian random vector).
The NTK matrix is defined recursively as \citep{jacot2018neural,arora2019exact}: For $l\in [L]$,
$$\Sigma^{(0)}(x, x') = x^Tx' \in \mathbb{R},$$
$$\Lambda^{(l)}(x, x') = \begin{pmatrix}
\Sigma^{(l-1)}(x, x) & \Sigma^{(l-1)}(x, x') \\
\Sigma^{(l-1)}(x, x')  & \Sigma^{(l-1)}(x', x') 
\end{pmatrix}\in \mathbb{R}^{2\times 2},$$
$$\Sigma^{(l)}(x, x') = 2 \mathbb{E}_{
(u,v)\sim \mathcal{N}(0,\Lambda(l))} [\sigma(u)\sigma(v)] \in \mathbb{R}.$$
We also define a derivative covariance:
$$\Sigma^{(l)}{'}(x, x') = 2 \mathbb{E}_{
(u,v)\sim \mathcal{N}(0,\Lambda(l))} [\sigma'(u)\sigma'(v)] \in \mathbb{R},$$
The final (population) NTK matrix is defined as
$$K(x, x') =\sum_{l=1}^{L+1}\left(\Sigma^{(l-1)}(x, x') \prod_{s=l}^{L+1}\Sigma^{(s)}{'}(x, x')\right)$$
where $\Sigma^{(L+1)}{'}(x, x')=1$ for convenience. Let $\langle \cdot, \cdot\rangle$ be the standard inner product in $\mathbb{R}^p$. Denote $J(\theta;x):=\nabla_{\theta} h_{\theta}(x)$. The empirical (since the initialization is random) NTK matrix is defined as
$K_{\theta}(x, x') =\langle J(\theta;x); J(\theta;x')\rangle.$
One of the most important statement in the NTK theory is that the empirical NTK matrix converges to the population NTK matrix as the width of the network increases
\citep{jacot2018neural,yang2019scaling,yang2020tensor,arora2019exact}. In practice we can use the empirical NTK matrix to approximate the population NTK matrix.

%We assume that the minimum eigenvalue of $K$, $\lambda_0:= \lambda_{\min}(K)$, is strictly positive which is a very mild assumption.

Based on NTK, we can derive that the neural network (NN) predictor is essentially a kernel regression (KR) predictor whose kernel is the NTK kernel. More precisely, we focus on the regression problem and consider the following regularized square loss function
\begin{equation} \label{equ:lossnn}
\hat{R}(h_\theta)=\frac{1}{n}\sum_{i=1}^n (h_\theta(x_i)-y_i)^2+ \lambda \|\theta-\theta_0\|^2_2.    
\end{equation}
We remark that in the theory of NTK, the training manner is typically the standard gradient descent/gradient flow (feeding the network with the entire training data) \citep{du2019gradient,lee2019wide,arora2019exact,zou2020gradient} so the uncertainty from the ordering of the data is of less importance. Therefore, we regard the random initialization $\Gamma_0$ as the only uncertainty from the procedural variability $\Gamma$.

Suppose the network is trained under the gradient flow: We minimize the empirical risk $\hat{R}(h_\theta)$ by gradient descent with infinitesimally small learning rate: $\frac{d\theta(t)}{dt} = -\nabla \hat{R}(h_\theta)$. 

Let $h(\theta(t); x)$ be the network output at time $t$. Then after training, the final output of the network is
\begin{equation} \label{equ:nnpredictor}
h_{nn}(\hat{\theta}; x) = \lim_{t\to \infty} h(\theta(t); x).
\end{equation}
Note that there is procedural variability in $\hat{\theta}$ because $\theta_0=\theta(0)$ is random. 

On the other hand, we build a kernel regression based on the NTK. Recall that we let $\bm{K}=(K(x_i,x_j))_{i,j=1,...,n} \in \mathbb{R}^{n\times n}$ be the NTK Gram matrix evaluated
on these training data. %Denote the minimum eigenvalue of $\bm{K}$, $\lambda_0:= \lambda_{\min}(\bm{K})$. 
For a test point $x_0 \in \mathbb{R}^d$, we let $K(x_0, \bm{x}) \in \mathbb{R}^{n}$ be the kernel value evaluated between the testing point and all training points, i.e., $K(x_0, \bm{x}):= (K(x_0, x_1), K(x_0, x_2), ... , K(x_0, x_n))^T$. The prediction of the kernel regression using NTK with respect to an initial $h(\theta_0; \cdot)$ is
\begin{align} \label{equ:krpredictor}
h_{kr}(x_0)= h(\theta_0; x_0)+
K(x_0, \bm{x})^T (\bm{K}+\lambda n \bm{I})^{-1} (\bm{y}-h(\theta_0;\bm{x}))
\end{align} 
In other words, $h_{kr}(x)-h(\theta_0; x)$ is the solution to the kernel ridge regression problem:
\begin{equation} \label{equ:losskr}
\argmin{g\in \bar{\mathcal{H}}} \frac{1}{n}\sum_{i=1}^n (y_i-h(\theta_0; x_i)-g(x_i))^2 + \lambda \|g\|^2_{\bar{\mathcal{H}}}
\end{equation}
where $\lambda$ is a regularization hyper-parameter and $\bar{\mathcal{H}}$ is the reproducing kernel Hilbert space constructed from the NTK kernel $K$ \citep{berlinet2011reproducing}. 

Note that there is procedural variability in $h_{kr}(x)$ because $\theta_0$ is random. %\citet{arora2019exact,lee2019wide,Hu2020Simple,he2020bayesian,zhang2020type} 

Previous work has established some equivalence between the NN predictor $h_{nn}(x)$ and the KR predictor $h_{kr}(x)$ based on a wide variety of assumptions; See below. We do not intend to dig into those detailed theorems and assume for simplicity that the following result holds:
\begin{assumption} \label{assu1}
Let $d_{\min}=\min\{d_1,...,d_L\}$. We have 
\begin{equation} \label{equ:NN=KR}
\lim_{d_{\min}\to \infty} \mathbb{E}_{\Gamma_0}[h_{nn}(\hat{\theta};x)]= \mathbb{E}_{\Gamma_0}[h_{kr}(x)].    
\end{equation}
where $\mathbb{E}_{\Gamma_0}$ is taken with respect to the random initialization $\Gamma_0$.
\end{assumption}

%Based on this assumption, one direction to estimate the epistemic variance \eqref{equ:var} is to analyze the variance of a KR predictor, which can be done via influence functions. This approach requires the evaluation of the NTK Gram matrix and its inversion. In addition, we propose an alternative approach based on the batching idea that can avoid the Gram matrix inversion.

In the case of $\lambda=0$ (no regularization term in the loss), \citet{arora2019exact,lee2019wide} have established the equivalence between the NN predictor and the KR predictor, as described below.

In \citet{arora2019exact}, their KR predictor is given by
\begin{align} \label{equ:krpredictornew}
h_{kr}(x_0)= K(x_0, \bm{x})^T (\bm{K}+\lambda n \bm{I})^{-1} \bm{y}
\end{align} 
without the initial term $h(\theta_0; x_0)$.
This is the standard kernel ridge regressor, which corresponds to kernel regression under gradient flow with zero initialization. Then they apply a small multiplier $\kappa > 0$, and modify the final output of the neural network to be 
\begin{equation} \label{equ:nnpredictor2}
h'_{nn}(x) = \kappa h_{nn}(\hat{\theta}; x).
\end{equation}
Their Theorem 3.2 rigorously establishes the equivalence between the (modified) NN predictor $h'_{nn}$ \eqref{equ:nnpredictor2} and the KR predictor $h_{kr}$ \eqref{equ:krpredictornew} using the NTK. Informally, with high probability with respect to the random initialization $\Gamma_0$, we have
$h_{kr}(x) \approx h'_{nn}(x).$

Simultaneously, \citet{lee2019wide} introduces the \textit{linearized neural network}, i.e,  they replace
the outputs of the neural network by their first order Taylor expansion. They obtain the final output of the linearized neural network (meaning $t\to \infty$ in the gradient descent) as
\begin{equation} \label{equ:linpredictor}
h^{lin}(\hat{\theta};x) := h(\theta_0; x) + K(x_0, \bm{x})^T \bm{K}^{-1} (\bm{y}-h(\theta_0; \bm{x}))
\end{equation}
which corresponds to our \eqref{equ:krpredictor} when $\lambda=0$. %; See Theorem \ref{thm:linear} for more details. 
They show that in the infinite width limit (which is less realistic than \citet{arora2019exact}), we have informally
$$h^{lin}(\hat{\theta};x)\approx h_{nn}(\hat{\theta};x).$$
%Note that this derivation also coincides with the output of the linearized neural network \citep{lee2019wide}:
In the case of $\lambda\ge 0$, \citet{he2020bayesian,Hu2020Simple} consider to add the weight regularization in the both loss functions, i.e., the $\lambda$ in \eqref{equ:lossnn} and \eqref{equ:losskr}. %They show the result similar to \eqref{equ:NN=KR}. 
%Let 
%$h_{nn}^\lambda(x)$ be the final output of the NN predictor and $h_{kr}^\lambda(x)$ be the KR predictor with this $\lambda$:
%\begin{align} \label{equ:krpredictor2}
%h_{kr}(\theta_0;x_0) := h(\theta_0; x_0) + K(x_0, \bm{x})^T (\bm{K}+\lambda n \bm{I})^{-1} (\bm{y}-h(\theta_0; \bm{x})).
%\end{align}
These papers have shown certain equivalence between the NN predictor $h_{nn}(x)$ and the KR predictor $h_{kr}(x)$, which is similar to \eqref{equ:NN=KR}. They adopt a wide variety of assumptions, such as infinite width of networks, or zero-valued initialization $h(\theta_0;\cdot)$, or a small multiplier $\kappa$, or normalized training inputs $\|x_i\|_2=1$. Their results are also of different forms. The one we adopt in this paper is close to \cite{lee2019wide} but with a weight regularization \citep{he2020bayesian}. Since a uniform statement is not available, we do not intend to write down a specific theorem for \eqref{equ:NN=KR} in this paper. Interested readers can refer to the above papers.

%In Sections \ref{sec: NTK} and \ref{sec: IF}, we will use $h_\theta(x)=h(\theta;x)$ as a shorthand for $h(\theta, \hat{\pi}_{\mathcal{D}_{tr}};x)$ since $\hat{\pi}_{\mathcal{D}_{tr}}$ remains the same throughout those sections. %In Section \ref{sec: NTK}, we will use $h_\theta(x)=h(\theta;x)$ as a shorthand for $h(\theta, \hat{\pi}_{\mathcal{D}_{tr}};x)$ since $\hat{\pi}_{\mathcal{D}_{tr}}$ remains the same throughout this section. 

\section{Proofs} \label{sec: proofs}

In this section, we provide detailed proofs for results in the main paper.

\begin{proof}[Proof of Proposition \ref{thm:modeluncertainty}]
For any $h\in \mathcal{H}$, we have by the definition of $f^*$
$$\mathbb{E}_{Y \sim  \pi_{Y|X}} [\mathcal{L}(f^*(X), Y)|X]\le \mathbb{E}_{Y \sim  \pi_{Y|X}} [\mathcal{L}(h(X), Y)|X]$$
Taking expectation with respect to $X$, we obtain
\begin{align*}
&\mathbb{E}[\mathcal{L}(f^*(X), Y)] \\
= &\mathbb{E}[\mathbb{E}_{Y \sim  \pi_{Y|X}} [\mathcal{L}(f^*(X), Y)|X]] \\
\le &\mathbb{E}[\mathbb{E}_{Y \sim  \pi_{Y|X}} [\mathcal{L}(h(X), Y)|X]]\\
= &\mathbb{E}[\mathcal{L}(h(X), Y)].
\end{align*}
So $f^*$ is the true risk minimizer in \eqref{equ:trueminimizer}, i.e., $h_{\theta^*}=f^*$.
\end{proof}

\begin{proof}[Proof of Theorem \ref{thm: decomp}]
We first note that
$$\text{UQ}_{MU}=h_{\theta^*}-f^*$$
which is a deterministic value. So $\text{UQ}_{MU}$ is uncorrelated with the $\text{UQ}_{PV}$ and $\text{UQ}_{DV}$ and
$$\text{Var} (\text{UQ}_{MU}+\text{UQ}_{PV}+\text{UQ}_{DV})=\text{Var} (\text{UQ}_{PV}+\text{UQ}_{DV}).$$
Furthermore, we have 
\begin{align*}
\text{Cov}(\text{UQ}_{PV},\text{UQ}_{DV})&=\text{Cov}(\mathbb E[\text{UQ}_{PV}|\hat{\pi}_{\mathcal{D}_{tr}}],\mathbb E[\text{UQ}_{DV}|\hat{\pi}_{\mathcal{D}_{tr}}])+\mathbb E[\text{Cov}(\text{UQ}_{PV},\text{UQ}_{DV}|\hat{\pi}_{\mathcal{D}_{tr}})]\\
&=0+0=0
\end{align*}
% \mathbb{E}[\mathbb{E}[\text{UQ}_{PV}\text{UQ}_{DV}|\hat{\pi}_{\mathcal{D}_{tr}}]]=\mathbb{E}[\text{UQ}_{DV}\mathbb{E}[\text{UQ}_{PV}|\hat{\pi}_{\mathcal{D}_{tr}}]]\\
% &=\mathbb{E}[\text{UQ}_{DV}]\mathbb{E}[\text{UQ}_{PV}|\hat{\pi}_{\mathcal{D}_{tr}}]=\mathbb{E}[\text{UQ}_{DV}]\mathbb{E}[\text{UQ}_{PV}]    
% \end{align*}
% \begin{align*}
% \mathbb{E}[\text{UQ}_{PV}\text{UQ}_{DV}]&=\mathbb{E}[\mathbb{E}[\text{UQ}_{PV}\text{UQ}_{DV}|\hat{\pi}_{\mathcal{D}_{tr}}]]=\mathbb{E}[\text{UQ}_{DV}\mathbb{E}[\text{UQ}_{PV}|\hat{\pi}_{\mathcal{D}_{tr}}]]\\
% &=\mathbb{E}[\text{UQ}_{DV}]\mathbb{E}[\text{UQ}_{PV}|\hat{\pi}_{\mathcal{D}_{tr}}]=\mathbb{E}[\text{UQ}_{DV}]\mathbb{E}[\text{UQ}_{PV}]    
% \end{align*}
where the first equality uses conditioning, and the second equality is because $\mathbb{E}[\text{UQ}_{PV}|\hat{\pi}_{\mathcal{D}_{tr}}]$, and also $\text{UQ}_{DV}$ given $\hat\pi_{\mathcal D_{tr}}$, are constant. This shows that $\text{UQ}_{PV}$ and $\text{UQ}_{DV}$ are uncorrelated, so
\begin{align*}
&\text{Var}(\text{UQ}_{PV}+\text{UQ}_{DV})=\text{Var}(\text{UQ}_{PV})+\text{Var}(\text{UQ}_{DV})\\
= &\text{Var}(\mathbb{E}[\text{UQ}_{PV}|\hat{\pi}_{\mathcal{D}_{tr}}])+\mathbb{E}[\text{Var}(\text{UQ}_{PV}|\hat{\pi}_{\mathcal{D}_{tr}})]
+\text{Var}(\mathbb{E}[\text{UQ}_{DV}|\hat{\pi}_{\mathcal{D}_{tr}}])+\mathbb{E}[\text{Var}(\text{UQ}_{DV}|\hat{\pi}_{\mathcal{D}_{tr}})]\\
= & 0+\mathbb{E}[\text{Var}(\text{UQ}_{PV}|\hat{\pi}_{\mathcal{D}_{tr}})]
+\text{Var}(\mathbb{E}[\text{UQ}_{DV}|\hat{\pi}_{\mathcal{D}_{tr}}])+0\\
= & \mathbb{E}[\text{Var}(\text{UQ}_{PV}|\hat{\pi}_{\mathcal{D}_{tr}})]
+\text{Var}(\mathbb{E}[\text{UQ}_{DV}|\hat{\pi}_{\mathcal{D}_{tr}}])
\end{align*}
where the second equality follows from conditioning, and the third equality is again because $\mathbb{E}[\text{UQ}_{PV}|\hat{\pi}_{\mathcal{D}_{tr}}]$, and $\text{UQ}_{DV}$ given $\hat\pi_{\mathcal D_{tr}}$, are constant. 
% :
% \begin{align*}
% &\text{Var}[V]=\mathbb{E}[V^2]-(\mathbb{E}[V])^2=\mathbb{E}[V^2]-(\mathbb{E}[\mathbb{E}[V|\mathcal{F}_0]])^2\\
% = & \mathbb{E}[\mathbb{E}[V^2|\mathcal{F}_0]-(\mathbb{E}[V|\mathcal{F}_0])^2]+\mathbb{E}[(\mathbb{E}[V|\mathcal{F}_0])^2]-(\mathbb{E}[\mathbb{E}[V|\mathcal{F}_0]])^2\\
% = & \mathbb{E}[\text{Var}(V|\mathcal{F}_0)]+\text{Var}(\mathbb{E}[(V|\mathcal{F}_0])
% \end{align*}
% for any $L^2$-integrable random variable $V$ and any sub-$\sigma$-field $\mathcal{F}_0$.

%\label{equ:var1}
%&=&\text{Var}( \mathbb{E}[\text{UQ}_{DV}|\hat{\pi}_{\mathcal{D}_{tr}}])+\mathbb{E}[\text{Var} (\text{UQ}_{PV}|\hat{\pi}_{\mathcal{D}_{tr}})]\label{equ:var2}
%\approx & \text{Var}( \mathbb{E}[h(\Gamma, \hat{\pi}_{\mathcal{D}_{tr}};x_0)|\hat{\pi}_{\mathcal{D}_{tr}}]) +\text{Var} (h(\Gamma, \pi;x_0))\\

\end{proof}

\begin{proof}[Proof of Theorem \ref{thm:var}]
In terms of estimation bias, we note that
$$\mathbb{E}[h_{\hat{\theta}_1}(x_0)|\hat{\pi}_{\mathcal{D}_{tr}}]=\mathbb{E}[h(\Gamma, \hat{\pi}_{\mathcal{D}_{tr}};x_0)|\hat{\pi}_{\mathcal{D}_{tr}}].$$
$$
\mathbb{E}[\bar{h}_{m'}(x_0)|\hat{\pi}_{\mathcal{D}_{tr}}]=\frac{1}{m'} \sum_{i=1}^{m'} \mathbb{E}[h_{\hat{\theta}_i}(x_0)|\hat{\pi}_{\mathcal{D}_{tr}}]
=\mathbb{E}[h(\Gamma, \hat{\pi}_{\mathcal{D}_{tr}};x_0)|\hat{\pi}_{\mathcal{D}_{tr}}].
$$
This implies that the deep ensemble prediction and the single
model prediction are the same. In terms of variance, note that $h_{\hat{\theta}_1}(x_0),...,h_{\hat{\theta}_{m'}}(x_0)$ are conditionally independent given $\hat{\pi}_{\mathcal{D}_{tr}}$. Therefore deep ensemble gives
\begin{align*}
\text{Var}(\bar{h}_{m'}(x_0)|\hat{\pi}_{\mathcal{D}_{tr}})=\frac{1}{(m')^2} \sum_{i=1}^{m'} \text{Var}(h_{\hat{\theta}_i}(x_0)|\hat{\pi}_{\mathcal{D}_{tr}})=\frac{1}{m'} \text{Var}(h(\Gamma, \hat{\pi}_{\mathcal{D}_{tr}};x_0)|\hat{\pi}_{\mathcal{D}_{tr}})
\end{align*}
while the single model prediction gives $\text{Var}(h(\Gamma,\hat{\pi}_{\mathcal{D}_{tr}};x_0)|\hat{\pi}_{\mathcal{D}_{tr}})$. Now using conditioning, we have 
\begin{align*}
&\text{Var} (\bar{h}_{m'}(x_0))\\
= & \text{Var}( \mathbb{E}[\bar{h}_{m'}(x_0))|\hat{\pi}_{\mathcal{D}_{tr}}]) +\mathbb{E}[\text{Var} (\bar{h}_{m'}(x_0)|\hat{\pi}_{\mathcal{D}_{tr}})]\\
= & \text{Var}(\mathbb{E}[h(\Gamma, \hat{\pi}_{\mathcal{D}_{tr}};x_0)|\hat{\pi}_{\mathcal{D}_{tr}}]) + \mathbb{E}[\frac{1}{m'} \text{Var}(h(\Gamma, \hat{\pi}_{\mathcal{D}_{tr}};x_0)|\hat{\pi}_{\mathcal{D}_{tr}})]\\
\le & \text{Var}(\mathbb{E}[h(\Gamma, \hat{\pi}_{\mathcal{D}_{tr}};x_0)|\hat{\pi}_{\mathcal{D}_{tr}}]) +\mathbb{E}[ \text{Var}(h(\Gamma, \hat{\pi}_{\mathcal{D}_{tr}};x_0)|\hat{\pi}_{\mathcal{D}_{tr}})]\\
= & \text{Var}(h(\Gamma, \hat{\pi}_{\mathcal{D}_{tr}};x_0))\\
= & \text{Var}(h_{\hat{\theta}_1}(x_0))
\end{align*}
as desired.
\end{proof}

%Our first observation is that
To prove Theorem \ref{thm:IF}, we need the following fact:
\begin{lemma} \label{thm:losskr}
$\bar{h}(\hat{\pi}_{\mathcal{D}_{tr}};\cdot)$ in \eqref{equ:KRR} is the solution to the following kernel ridge regression problem
\begin{equation} \label{equ:losskr2}
\bar{h}-h_0= \argmin{g\in \bar{\mathcal{H}}} \frac{1}{n}\sum_{i=1}^n (y_i-h_0(x_i)-g(x_i))^2 + \lambda \|g\|^2_\mathcal{\bar{H}}.
\end{equation}
\end{lemma}

\begin{proof}[Proof of Lemma \ref{thm:losskr}]
We first let $\tilde{y}_i=y_i-h_0(x_i)$ be the shifted label. With a little abuse of notation, $\hat{\pi}_{\mathcal{D}_{tr}}$ should be understood as the empirical distribution on the shifted training data $\{(x_i, \tilde{y}_i): i=1,...,n\}$. Now consider the kernel ridge regression problem
$$\argmin{g\in \bar{\mathcal{H}}} \frac{1}{n}\sum_{i=1}^n (\tilde{y}_i-g(x_i))^2 + \lambda \|g\|^2_\mathcal{\bar{H}}.$$
Standard theory in kernel ridge regression (e.g., Theorem 1 in \citet{smale2005shannon}) shows that the closed form of the solution to the above problem is
$$g^*(x_0) = K(x_0, \bm{x})^T(\bm{K}+\lambda n \bm{I})^{-1} \bm{\tilde{y}}= K(x_0, \bm{x})^T(\bm{K}+\lambda n \bm{I})^{-1} (\bm{y}-h_0(\bm{x}))$$
which corresponds to the definition of $\bar{h}(\hat{\pi}_{\mathcal{D}_{tr}};x_0)-h_0(x_0)$.
\end{proof}

\begin{proof}[Proof of Theorem \ref{thm:IF}]
Let $\bar{h}(x)=\bar{h}(\hat{\pi}_{\mathcal{D}_{tr}};x)$ for short. We first note that $\text{Var}(\bar{h}(x_0)-h_0(x_0))=\text{Var}(\bar{h}(x_0))$ where the variance is only respect to the training data. Hence it is sufficient to study the variance of $\bar{h}(x_0)-h_0(x_0)$ which is the solution of the problem \eqref{equ:losskr2} in Lemma \ref{thm:losskr}.

Let $\Phi$ be the feature map associated with the NTK kernel. We apply Proposition 5 in \citet{debruyne2008model} (see also \citet{christmann2007consistency}) to calculate the influence function. With a slight abuse of notation, we use $IF(z; T, \hat{\pi}_{\mathcal{D}_{tr}})(x_0)$ to represent the influence function at the fixed test point $x_0$ in the main paper, and reserve $IF(z; T, \hat{\pi}_{\mathcal{D}_{tr}})$ as a function only in this proof to be consistent with \citet{debruyne2008model}.
\begin{align*}
IF(z; T, \hat{\pi}_{\mathcal{D}_{tr}}) &= -S^{-1}(2\lambda (\bar{h}-h_0)) + \mathcal{L}'(z_{y}-h_0(z_x)- (\bar{h}(z_x)-h_0(z_x))) S^{-1}\Phi (z_x)   \\
&= -S^{-1}(2\lambda (\bar{h}-h_0)) + 2(z_{y}- \bar{h}(z_x)) S^{-1}\Phi (z_x) 
\end{align*}
where $S : \mathcal{H} \to H$ is defined by $$S(f) = 2\lambda f + \mathbb{E}_{\hat{\pi}_{\mathcal{D}_{tr}}} [\mathcal{L}''(Y - \bar{h}(X))\langle \Phi(X), f\rangle  \Phi(X)]$$

To do this, we need to obtain $S$ and $S^{-1}$. Since the loss function is $\mathcal{L}(\hat{y}-y)=(\hat{y}-y)^2$, we have
$$S(f) = 2\lambda f + \mathbb{E}_{\hat{\pi}_{\mathcal{D}_{tr}}} [\mathcal{L}''(Y - \bar{h}(X))\langle \Phi(X), f\rangle  \Phi(X)]= 2\lambda f + \frac{2}{n} \sum_{j=1}^{n} f(x_j)\Phi(x_j ).$$
Suppose $S^{-1}(2\lambda (\bar{h}-h_0))=g_1$. Then at $x_0$, we have
$$2\lambda (\bar{h}(x_0)-h_0(x_0))=S(g_1(x_0))=2\lambda g_1(x_0) + \frac{2}{n} \sum_{j=1}^{n} g_1(x_j)K(x_0,x_j).$$
Hence
$$g_1(x_0)=\bar{h}(x_0)-h_0(x_0)-\frac{1}{\lambda n} \sum_{j=1}^n g_1(x_j)K(x_0,x_j)$$
This implies that we need to evaluate $g_1(x_j)$ on training data first, which is straightforward by letting $x_0=x_1,...,x_n$:
$$g_1(\bm{x})=\bar{h}(\bm{x})-h_0(\bm{x})-\frac{1}{\lambda n} \bm{K} g_1(\bm{x})$$
so 
$$g_1(\bm{x})=(\bm{K}+\lambda n \bm{I})^{-1}(\lambda n)\left(\bar{h}(\bm{x})-h_0(\bm{x})\right)$$
and
\begin{align*}
g_1(x_0)&=\bar{h}(x_0)-h_0(x_0)-\frac{1}{\lambda n} K(x_0, \bm{x})^T g_1(\bm{x})    \\
&=\bar{h}(x_0)-h_0(x_0)-K(x_0, \bm{x})^T (\bm{K}+\lambda n \bm{I})^{-1}\left(\bar{h}(\bm{x})-h_0(\bm{x})\right)
\end{align*}
Next we compute $S^{-1}\Phi (z_x)=g_2$. At $x_0$, we have
$$K(z_x,x_0)=\Phi (z_x)(x_0)=S(g_2(x_0))=2\lambda g_2(x_0) + \frac{2}{n} \sum_{j=1}^{n} g_2(x_j)K(x_0,x_j)$$
Hence 
$$g_2(x_0)=\frac{1}{2\lambda} K(z_x,x_0)-\frac{1}{\lambda n} \sum_{j=1}^n g_2(x_j)K(x_0,x_j)$$
This implies that we need to evaluate $g_2(x_j)$ on training data first, which is straightforward by letting $x_0=x_1,...,x_n$:
$$g_2(\bm{x})=\frac{1}{2\lambda} K(z_x,\bm{x})-\frac{1}{\lambda n} \bm{K} g_2(\bm{x})$$
so 
$$g_2(\bm{x})=(\bm{K}+\lambda n \bm{I})^{-1}(\frac{n}{2})K(z_x,\bm{x})$$
and
\begin{align*}
g_2(x_0)&=\frac{1}{2\lambda} K(z_x,x_0)-\frac{1}{\lambda n} K(x_0,\bm{x})^T g_2(\bm{x})    \\
&=\frac{1}{2\lambda} K(z_x,x_0)-\frac{1}{2\lambda} K(x_0, \bm{x})^T (\bm{K}+\lambda n \bm{I})^{-1}K(z_x,\bm{x})
\end{align*}
Combing previous results, we obtain
\begin{align*}
IF(z; T, \hat{\pi}_{\mathcal{D}_{tr}})(x_0)=&-\bar{h}(x_0)+h_0(x_0)+K(x_0, \bm{x})^T (\bm{K}+\lambda n \bm{I})^{-1}\left(\bar{h}(\bm{x})-h_0(\bm{x})\right)\\
&+2(z_{y}- \bar{h}(z_x)) \left(\frac{1}{2\lambda} K(z_x,x_0)-\frac{1}{2\lambda} K(x_0, \bm{x})^T (\bm{K}+\lambda n \bm{I})^{-1}K(z_x,\bm{x})\right)\\
=&K(x_0, \bm{x})^T (\bm{K}+\lambda n \bm{I})^{-1} \left(\bar{h}(\bm{x})-h_0(\bm{x})-\frac{1}{\lambda}(z_{y}- \bar{h}(z_x)) K(z_x,\bm{x})\right)\\
&-\bar{h}(x_0)+h_0(x_0)+\frac{1}{\lambda}(z_{y}-\bar{h}(z_x)) K(z_x,x_0)
\\
=& K(x_0, \bm{x})^T (\bm{K}+\lambda n \bm{I})^{-1} M_z(\bm{x})-M_z(x_0)
\end{align*}
as desired.
\end{proof}

\begin{proof}[Proof of Theorem \ref{thm:IF2}]
In this proof, we write $\hat{\pi}^n_{\mathcal{D}_{tr}}$ rather than $\hat{\pi}^n_{\mathcal{D}_{tr}}$ to emphsize the size of training dataset $n$.
Since $T(P)$ is $\rho_\infty$-Hadamard differentiable at $\pi$, we have the following asymptotic normality by the central limit theorem \citep{fernholz2012mises,shao2012jackknife}:
$$\sqrt{n} (T(\hat{\pi}^n_{\mathcal{D}_{tr}})-T(\pi))\xrightarrow{\text{d}} \mathcal{N}(0,\sigma^2)$$
where $\sigma^2=\mathbb{E}_{z\sim \pi} [IF^2(z; T, \pi)]$. In other words, 
$$\sqrt{n} (\bar{h}(\hat{\pi}^n_{\mathcal{D}_{tr}};x_0)-h_{\theta^*}(x_0))\xrightarrow{\text{d}} \mathcal{N}(0,\sigma^2)$$
The uniform integrability implies (by Skorokhod's theorem and Vitali convergence theorem) that
$$\mathbb{E}\Big[\Big(\sqrt{n} (\bar{h}(\hat{\pi}^n_{\mathcal{D}_{tr}};x_0)-h_{\theta^*}(x_0))\Big)^2\Big]\to \mathbb{E}\Big[\mathcal{N}(0,\sigma^2)^2\Big]=\sigma^2$$
which gives that
\begin{equation} \label{4.3equ0}
\lim_{n\to\infty} |n \text{Var} (\bar{h}(\hat{\pi}^n_{\mathcal{D}_{tr}};x_0))- \sigma^2|=0    
\end{equation}

Since $T(P)$ is continuously $\rho_\infty$-Hadamard differentiable (which implies continuously G\^{a}teaux differentiable) at $\hat{\pi}^n_{\mathcal{D}_{tr}}$, we have that
$$\lim_{t\to 0} \left[\frac{T(\hat{\pi}^n_{\mathcal{D}_{tr}}+t(\delta_{z}-T(\hat{\pi}^n_{\mathcal{D}_{tr}}))-T(\hat{\pi}^n_{\mathcal{D}_{tr}})}{t}-IF(z; T, \hat{\pi}^n_{\mathcal{D}_{tr}})\right]=0, \quad \text{ uniformly in } z$$
where implies that for every $n$, there exists a $t_n$ (depending on $\hat{\pi}^n_{\mathcal{D}_{tr}}$), such that
\begin{equation} \label{4.3equ1}
\left|\frac{T(\hat{\pi}^n_{\mathcal{D}_{tr}}+t_n(\delta_{z}-T(\hat{\pi}^n_{\mathcal{D}_{tr}}))-T(\hat{\pi}^n_{\mathcal{D}_{tr}})}{t_n}-IF(z; T, \hat{\pi}^n_{\mathcal{D}_{tr}})\right|\le \epsilon, \quad \text{ uniformly in } z.
\end{equation}
Without loss of generality, we assume $t_n\to 0$ as $n\to \infty$.

Again, note that $T(P)$ is continuously $\rho_\infty$-Hadamard differentiable (which implies continuously G\^{a}teaux differentiable) at $\pi$. Then since $t_n\to 0$ and $\rho_\infty(\hat{\pi}^n_{\mathcal{D}_{tr}},\pi)\xrightarrow{a.s.} 0$, the definition implies that
$$\lim_{n\to \infty} \left[\frac{T(\hat{\pi}^n_{\mathcal{D}_{tr}}+t_n(\delta_{z}-T(\hat{\pi}^n_{\mathcal{D}_{tr}}))-T(\hat{\pi}^n_{\mathcal{D}_{tr}})}{t_n}-L_\pi(\delta_z-\hat{\pi}^n_{\mathcal{D}_{tr}})\right]=0, \quad a.s.\ \text{ uniformly in } z$$
where $L_\pi(\delta_z-\hat{\pi}^n_{\mathcal{D}_{tr}})$ is the G\^{a}teaux differential of $T$ at $\pi$. Therefore, there exists a $N_1$, such that for $n\ge N_1$,
\begin{equation} \label{4.3equ2}
\left|\frac{T(\hat{\pi}^n_{\mathcal{D}_{tr}}+t_n(\delta_{z}-T(\hat{\pi}^n_{\mathcal{D}_{tr}}))-T(\hat{\pi}^n_{\mathcal{D}_{tr}})}{t_n}-L_\pi(\delta_z-\hat{\pi}^n_{\mathcal{D}_{tr}})\right|\le \epsilon, \quad a.s.\ \text{ uniformly in } z.    
\end{equation}
As G\^{a}teaux differential is a linear functional, we have that
\begin{align}
L_\pi(\delta_z-\hat{\pi}^n_{\mathcal{D}_{tr}})&=L_\pi(\delta_z-\pi)+ L_\pi(\pi-\hat{\pi}^n_{\mathcal{D}_{tr}})\nonumber\\
&=L_\pi(\delta_z-\pi)+ \frac{1}{n}\sum_{i=1}^n L_\pi(\pi-\delta_{z_i})\nonumber\\
&=IF(z; T, \pi)-\frac{1}{n}\sum_{i=1}^{n} IF(z_i; T, \pi)  \label{4.3equ3}  
\end{align}
Therefore, combing \eqref{4.3equ1}, \eqref{4.3equ2} and \eqref{4.3equ3}, we have that for $n\ge N_1$,
\begin{equation*}
\left|IF(z; T, \hat{\pi}^n_{\mathcal{D}_{tr}})-\left(IF(z; T, \pi)-\frac{1}{n}\sum_{i=1}^{n} IF(z_i; T, \pi)\right)\right|\le 2\epsilon, \quad \text{uniformly in } z.    
\end{equation*}
Hence
\begin{align*}
&\left|IF^2(z; T, \hat{\pi}^n_{\mathcal{D}_{tr}})-\left(IF(z; T, \pi)-\frac{1}{n}\sum_{i=1}^{n} IF(z_i; T, \pi)\right)^2\right|\\
=&\left|IF(z; T, \hat{\pi}^n_{\mathcal{D}_{tr}})+\left(IF(z; T, \pi)-\frac{1}{n}\sum_{i=1}^{n} IF(z_i; T, \pi)\right)\right| \left|IF(z; T, \hat{\pi}^n_{\mathcal{D}_{tr}})-\left(IF(z; T, \pi)-\frac{1}{n}\sum_{i=1}^{n} IF(z_i; T, \pi)\right)\right|\\
\le& \epsilon\left(2\left|\left(IF(z; T, \pi)-\frac{1}{n}\sum_{i=1}^{n} IF(z_i; T, \pi)\right)\right|+ \epsilon\right)\\
=& 2\epsilon\left|\left(IF(z; T, \pi)-\frac{1}{n}\sum_{i=1}^{n} IF(z_i; T, \pi)\right)\right|+ \epsilon^2.
\end{align*}
Since this inequality holds for uniformly in $z$, take $z=z_j$ where $j \in [n]$, we obtain
\begin{align}
&\left|\frac{1}{n}\sum_{j=1}^n IF^2(z_j; T, \hat{\pi}^n_{\mathcal{D}_{tr}})-\frac{1}{n}\sum_{j=1}^n\left(IF(z_j; T, \pi)-\frac{1}{n}\sum_{i=1}^{n} IF(z_i; T, \pi)\right)^2\right|\nonumber\\
\le & 2\epsilon \frac{1}{n}\sum_{j=1}^n\left|\left(IF(z_j; T, \pi)-\frac{1}{n}\sum_{i=1}^{n} IF(z_i; T, \pi)\right)\right|+ \epsilon^2\nonumber\\
\le & 2\epsilon \sqrt{\frac{1}{n}\sum_{j=1}^n\left(IF(z_j; T, \pi)-\frac{1}{n}\sum_{i=1}^{n} IF(z_i; T, \pi)\right)^2}+ \epsilon^2 \label{4.3equ4}
\end{align}
Note that by the strong law of large number, we have that
\begin{align*}
&\frac{1}{n}\sum_{j=1}^n\left(IF(z_j; T, \pi)-\frac{1}{n}\sum_{i=1}^{n} IF(z_i; T, \pi)\right)^2\\
=&\frac{1}{n}\sum_{j=1}^n IF^2(z_j; T, \pi)-\left(\frac{1}{n}\sum_{i=1}^{n} IF(z_i; T, \pi)\right)^2\\
\xrightarrow{a.s.} & \sigma^2
\end{align*}
since $\mathbb{E}_{z\sim \pi} [IF^2(z; T, \pi)]=\sigma^2$ and $\mathbb{E}_{z\sim \pi} [IF(z; T, \pi)]=0$. Therefore \eqref{4.3equ4} implies that
\begin{align*}
\limsup_{n\to \infty}\left|\frac{1}{n}\sum_{j=1}^n IF^2(z_j; T, \hat{\pi}^n_{\mathcal{D}_{tr}})-\frac{1}{n}\sum_{j=1}^n\left(IF(z_j; T, \pi)-\frac{1}{n}\sum_{i=1}^{n} IF(z_i; T, \pi)\right)^2\right|\nonumber
\le 2\epsilon \sigma+ \epsilon^2, \quad a.s. 
\end{align*}
and thus
\begin{align*}
\lim_{n\to \infty}\left|\frac{1}{n}\sum_{j=1}^n IF^2(z_j; T, \hat{\pi}^n_{\mathcal{D}_{tr}})-\sigma^2\right|=0, \quad a.s. 
\end{align*}
Combing it with \eqref{4.3equ0}, we obtain
\begin{align*}
\lim_{n\to \infty}\left|\frac{1}{n}\sum_{j=1}^n IF^2(z_j; T, \hat{\pi}^n_{\mathcal{D}_{tr}})-n \text{Var} (\bar{h}(\hat{\pi}^n_{\mathcal{D}_{tr}};x_0))\right|=0, \quad a.s. 
\end{align*}
\end{proof}

%\begin{proof}[Proof of Theorem \ref{thm:EV}]
%Note that $\mathbb{E}[h(\Gamma, \hat{\pi}_{\mathcal{D}_{tr}};x_0)^2]< \infty$ implies that
%$$\mathbb{E}[\text{Var} (h(\Gamma, \hat{\pi}_{\mathcal{D}_{tr}};x_0)|\hat{\pi}_{\mathcal{D}_{tr}})]\le \mathbb{E}[h(\Gamma, \hat{\pi}_{\mathcal{D}_{tr}};x_0)^2]< \infty$$
%so by the Varadarajan's theorem (noting that $\mathbb{R}^{d+1}$ is a  separable metric space),
%$$\text{Var} (h(\Gamma, \hat{\pi}_{\mathcal{D}_{tr}};x_0)|\hat{\pi}_{\mathcal{D}_{tr}}) \to \mathbb{E}[\text{Var} (h(\Gamma, \hat{\pi}_{\mathcal{D}_{tr}};x_0)|\hat{\pi}_{\mathcal{D}_{tr}})], \quad a.s.$$ \end{proof}

\begin{proof}[Proof of Theorem \ref{DE main}]

Note that $K=m'$. We have, by Assumptions \eqref{normal} and \eqref{normal batch}, that
$$\hat{\psi}_i(x_0)-h_{\theta^*}(x_0) = (\sqrt{K}\nu Y_i+B_0) (1+o_p(1)), $$
$$\bar{\psi}(x_0)-h_{\theta^*}(x_0)= (\sqrt{K} \nu \bar{Y}+B_0)(1+o_p(1)),$$
where $Y_i$ are i.i.d. $\mathcal{N}(0,1)$, $\bar{Y}=\frac{1}{K}\sum_{i=1}^K Y_i$ and $\nu=\sqrt{\sigma^2/n+\tau^2/m'}=\sqrt{\sigma^2/n+\tau^2/K}$. Then, we note some basic facts that $\bar Y\sim \frac{1}{\sqrt{K}}\mathcal N(0,1)$, and $\sum_{i=1}^K(Y_i-\bar Y)^2\sim\chi^2_{K-1}$, and these two variables are independent of each other. Therefore by the continuous mapping theorem, we have
$$\frac{S}{\sqrt{K}}=\sqrt{\frac{ \sum_{i=1}^K(\hat{\psi}_i(x_0)-\bar{\psi}(x_0))^2}{K(K-1)}} \Rightarrow\sqrt{\frac{K \sum_{i=1}^K(\nu Y_i-\nu \bar{Y})^2}{K(K-1)}}=\sqrt{\frac{\sum_{i=1}^K(\nu Y_i-\nu \bar{Y})^2}{K-1}},$$
and thus
% $$\frac{\bar\psi(x_0)-h_{\theta^*}(x_0)}{S/\sqrt K}\Rightarrow\frac{\sqrt{K}\nu\bar Y}{\sqrt{\frac{1}{K-1}\sum_{i=1}^K(\nu Y_i-\nu \bar Y)^2}}\stackrel{d}{=}N(0,1)/\sqrt{\frac{\chi^2_{K-1}}{K-1}}=t_{K-1},$$
% and
$$\frac{S^2/K}{\nu^2}\Rightarrow\frac{ \sum_{i=1}^K(\nu Y_i-\nu \bar Y)^2}{\nu^2(K-1)}=\frac{1}{K-1}\sum_{i=1}^K(Y_i-\bar Y)^2=\frac{\chi^2_{K-1}}{K-1}.$$
%To prove \eqref{sectioning}, we observe that 
%$$\psi(\hat P)-\bar\psi=\xi-\frac{1}{K}\sum_{i=1}^K\xi_i=o_p\left(\frac{1}{\sqrt l}\right)$$
%which concludes the result.
%The assertion \eqref{var est} follows from a similar derivation as above that
The CI \eqref{var CI DE} then follows immediately from the above result, since it implies
$$P\left(\frac{\chi^2_{K-1,\alpha/2}}{K-1}\leq\frac{S^2/K}{\nu^2}\leq\frac{\chi^2_{K-1,1-\alpha/2}}{K-1}\right)\to1-\alpha$$or equivalently
$$P\left(\frac{S^2/K}{\chi^2_{K-1,1-\alpha/2}/(K-1)}\leq\nu^2\leq\frac{S^2/K}{\chi^2_{K-1,\alpha/2}/(K-1)}\right)\to1-\alpha$$
which concludes the result.
\end{proof}

\section{Experiments: Details and More Results} \label{sec:expmore}

\subsection{Experimental Details} \label{sec:expmore1}
We provide more details about obtaining the ground-truth data variance and procedural variance separately. As we discussed in Section \ref{sec:exp}, we have employed that
$$\text{Var}(h(x_0)) \approx \frac{1}{J-1}\sum_{j=1}^J \left(h_j(x_0) - \frac{1}{J}\sum_{i=1}^J h_i(x_0)\right)^2:= \tilde{\text{Var}}(h(x_0))$$
where $J$ is the number of experimental repetitions (typically vary large). Note that we can allow $h(x_0)$ to be a single model predictor $h_{\hat{\theta}_1}(x_0)$ or a deep ensemble predictor $\bar{h}_{m'}(x_0)$. It follows from Theorem \ref{thm:var} in Section \ref{sec: DE} that
$$\mathbb{E}[ \text{Var}(h(\Gamma, \hat{\pi}_{\mathcal{D}_{tr}};x_0)|\hat{\pi}_{\mathcal{D}_{tr}})]=\frac{m'}{m'-1}\left(\text{Var}(h_{\hat{\theta}_1}(x_0))-\text{Var} (\bar{h}_{m'}(x_0))\right),
$$
$$\text{Var}(\mathbb{E}[h(\Gamma, \hat{\pi}_{\mathcal{D}_{tr}};x_0)|\hat{\pi}_{\mathcal{D}_{tr}}])=\frac{m'}{m'-1}\text{Var} (\bar{h}_{m'}(x_0))-\frac{1}{m'-1} \text{Var}(h_{\hat{\theta}_1}(x_0)).
$$
Hence, by replacing $\text{Var}(h_{\hat{\theta}_1}(x_0)), \text{Var} (\bar{h}_{m'}(x_0))$ with $\tilde{\text{Var}} (h_{\hat{\theta}_1}(x_0)), \tilde{\text{Var}} (\bar{h}_{m'}(x_0))$ respectively in the above equations, we obtain the ground truth data variance and procedural variance.

Throughout our experiments, we use a two-layer fully-connected NN with $1024$ hidden neurons as the base predictor based on the NTK specification in Section \ref{sec: NTK} and it is trained using the regularized square loss \eqref{equ:lossnn} with regularization hyper-parameter $\lambda=0.1^3$. The learning rate is properly tuned based on the specific dataset and it is small (around $10^{-2}$) to make the training procedure operate in the NTK regime \citep{jacot2018neural,du2019gradient,lee2019wide,zhang2020type}. All experiments are conducted on a single GeForce RTX 2080 Ti GPU.

\subsection{Computational Analysis} \label{sec:CA}
We further provide computational analysis for our approaches in the experiments. We report the execution time of the widely-used benchmark Boston dataset as a reference. The execution time of our approaches is approximately (unit: minutes): EV: 23.4 $|$ IF: 12.4 $|$ BA: 0.9. 
In addition, we illustrate the curve of EV values vs ensemble times in Figure \ref{fig:boston}. This shows that running 50 ensemble times is sufficient to produce robust results and running less ensemble times (to reduce the running time) such as 30 is also feasible. As shown, our approaches are very efficient and even have the potential for real-time applications.  

\begin{figure}
    \centering
    \includegraphics[width=0.5\textwidth]{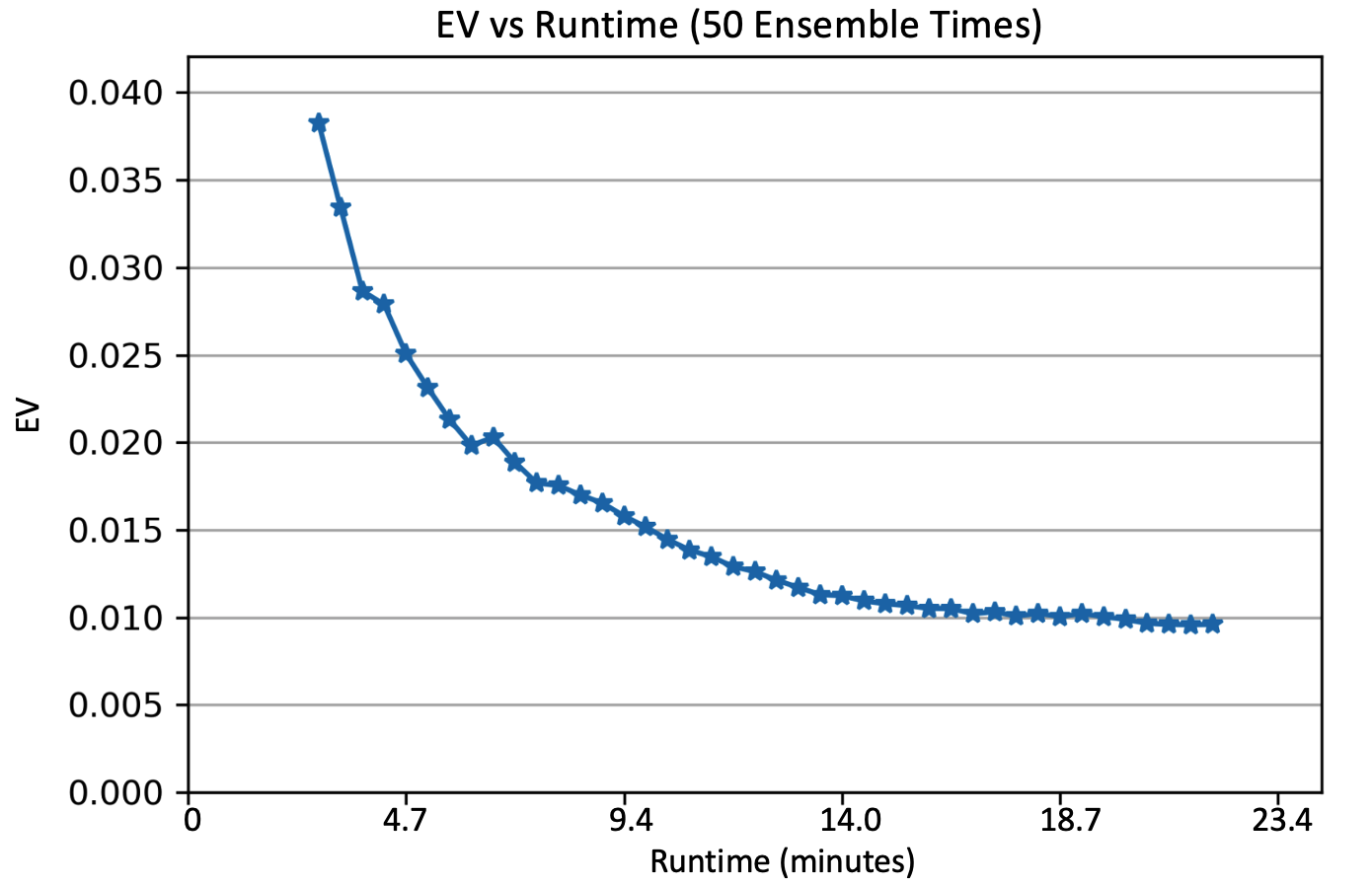}
    \caption{EV values vs ensemble times.}
    \label{fig:boston}
\end{figure}

\subsection{Additional Experiments} \label{sec:expmore2}
In this section, we present additional experimental results on more datasets.

Synthetic Datasets 2: $X \sim \mathcal{N}(0, 0.1^2\bm{I}_d)$ and 
$Y \sim \sum_{i=1}^d \exp(X^{(i)})+(X^{(i)})^2 + \mathcal{N}(0,(0.1d)^2)$.
The training set $\mathcal{D}_{tr}=\{(x_i,y_i): i=1,...,n\}$ are formed by i.i.d. samples of $(X,Y)$ with sample size $n=200$. We use $x_0=(0.1,0.1,...,0.1)$ as the fixed test point and set $K=m'=5$, $m=50$. The results are displayed in Table \ref{results2}. 

Synthetic Datasets 3: $X \sim \mathcal{N}(0, 0.1^2\bm{I}_d)$ and 
$Y \sim \sum_{i=1}^d \cos(X^{(i)})+(X^{(i)})^3 + \mathcal{N}(0,(0.1d)^2)$.
The training set $\mathcal{D}_{tr}=\{(x_i,y_i): i=1,...,n\}$ are formed by i.i.d. samples of $(X,Y)$ with sample size $n=500$. We use $x_0=(0.1,0.1,...,0.1)$ as the fixed test point and set $K=m'=5$, $m=50$. The results are displayed in Table \ref{results3}. 

%Synthetic Datasets 4: $X \sim \text{Unif}([0,0.2]^d)$ and $Y \sim \sum_{i=1}^d \sin(X^{(i)}) + \mathcal{N}(0,0.1^2)$. The training set $\mathcal{D}_{tr}=\{(x_i,y_i): i=1,...,n\}$ are formed by i.i.d. samples of $(X,Y)$ with sample size $n=200$. We use $x_0=(0.1,0.1,...,0.1)$ as the fixed test point and set $K=m'=5$, $m=50$. The results are displayed in Table \ref{results4}. 

%Synthetic Datasets 5: $X \sim \text{Unif}([0,0.2]^d)$ and $Y \sim \sum_{i=1}^d \sin(X^{(i)}) + \mathcal{N}(0,0.1^2)$. The training set $\mathcal{D}_{tr}=\{(x_i,y_i): i=1,...,n\}$ are formed by i.i.d. samples of $(X,Y)$ with sample size $n=500$. We use $x_0=(0.1,0.1,...,0.1)$ as the fixed test point and set $K=m'=5$, $m=50$. The results are displayed in Table \ref{results5}. 

%Synthetic Datasets 6: $X \sim \text{Unif}([0,0.2]^d)$ and $Y \sim \sum_{i=1}^d \sin(X^{(i)}) + \mathcal{N}(0,0.1^2)$. The training set $\mathcal{D}_{tr}=\{(x_i,y_i): i=1,...,n\}$ are formed by i.i.d. samples of $(X,Y)$ with sample size $n=1000$. We use $x_0=(0.1,0.1,...,0.1)$ as the fixed test point and set $K=m'=5$, $m=50$. The results are displayed in Table \ref{results6}. 

\begin{table}[ht] 
\small
  \centering
  \begin{tabular}{c|c|c|c}
    \toprule
      & $\tau^2$ GT &  $\tau^2$ EV & $\tau^2$ Diff\\
            \midrule
      $d=2$ & $0.1*10^{-3}$ &  $0.1*10^{-3}$ &  $-0.0*10^{-3}$\\
      $d=4$ & $1.7*10^{-3}$ &  $1.4*10^{-3}$ &  $-0.3*10^{-3}$\\
      $d=8$ & $1.8*10^{-2}$ &  $1.3*10^{-2}$ &  $-0.5*10^{-2}$\\
      $d=16$ & $1.4*10^{-1}$ &  $0.9*10^{-1}$ &  $-0.5*10^{-1}$\\
      \midrule
      &  $\frac{\sigma^2}{n}$ GT & $\frac{\sigma^2}{n}$ IF & $\frac{\sigma^2}{n}$ Diff \\ 
            \midrule
      $d=2$ & $1.1*10^{-3}$ & $1.5*10^{-3}$ & $+0.4*10^{-3}$\\
      $d=4$ & $8.1*10^{-3}$ & $8.7*10^{-3}$ & $+0.6*10^{-3}$\\
      $d=8$ & $4.8*10^{-2}$ & $5.0*10^{-2}$  & $+0.2*10^{-2}$\\
      $d=16$ & $3.0*10^{-1}$ & $3.4*10^{-1}$  & $+0.4*10^{-1}$\\    
      \midrule
      & $\frac{\sigma^2}{n}+\frac{\tau^2}{m'}$ GT & $\frac{\sigma^2}{n}+\frac{\tau^2}{m'}$ BA & $\frac{\sigma^2}{n}+\frac{\tau^2}{m'}$ Diff\\
            \midrule
      $d=2$ & $1.1*10^{-3}$ & $1.5*10^{-3}$ & $+0.4*10^{-3}$ \\
      $d=4$ & $8.4*10^{-3}$ & $8.6*10^{-3}$ & $+0.2*10^{-3}$ \\
      $d=8$ & $5.2*10^{-2}$ & $5.6*10^{-2}$  & $+0.4*10^{-2}$\\
      $d=16$ & $3.3*10^{-1}$ & $3.6*10^{-1}$  & $+0.3*10^{-1}$\\
    \bottomrule
  \end{tabular}
  \caption{Epistemic Variance Estimation on Synthetic Datasets 2 with Different Dimensions.}
 \label{results2}  
\end{table}

\begin{table}[ht] 
\small
  \centering
  \begin{tabular}{c|c|c|c}
    \toprule
      & $\tau^2$ GT &  $\tau^2$ EV & $\tau^2$ Diff\\
            \midrule
      $d=2$ & $4.5*10^{-4}$ &  $4.3*10^{-4}$ &  $-0.2*10^{-4}$\\
      $d=4$ & $1.1*10^{-3}$ &  $0.8*10^{-3}$ &  $-0.3*10^{-3}$\\
      $d=8$ & $1.0*10^{-2}$ &  $0.7*10^{-2}$ &  $-0.3*10^{-2}$\\
      $d=16$ & $0.4*10^{-1}$ &  $0.2*10^{-1}$ &  $-0.2*10^{-1}$\\
      \midrule
      &  $\frac{\sigma^2}{n}$ GT & $\frac{\sigma^2}{n}$ IF & $\frac{\sigma^2}{n}$ Diff \\ 
            \midrule
      $d=2$ & $5.7*10^{-4}$ & $6.3*10^{-4}$ & $+0.6*10^{-4}$\\
      $d=4$ & $3.9*10^{-3}$ & $4.5*10^{-3}$ & $+0.6*10^{-3}$\\
      $d=8$ & $2.1*10^{-2}$ & $2.4*10^{-2}$  & $+0.3*10^{-2}$\\
      $d=16$ & $0.6*10^{-1}$ & $0.8*10^{-1}$  & $+0.2*10^{-1}$\\    
      \midrule
      & $\frac{\sigma^2}{n}+\frac{\tau^2}{m'}$ GT & $\frac{\sigma^2}{n}+\frac{\tau^2}{m'}$ BA & $\frac{\sigma^2}{n}+\frac{\tau^2}{m'}$ Diff\\
            \midrule
      $d=2$ & $6.6*10^{-4}$ & $6.9*10^{-4}$ & $+0.3*10^{-4}$ \\
      $d=4$ & $4.1*10^{-3}$ & $4.3*10^{-3}$ & $+0.2*10^{-3}$ \\
      $d=8$ & $2.3*10^{-2}$ & $2.5*10^{-2}$  & $+0.2*10^{-2}$\\
      $d=16$ & $0.7*10^{-1}$ & $0.8*10^{-1}$  & $+0.1*10^{-1}$\\
    \bottomrule
  \end{tabular}
  \caption{Epistemic Variance Estimation on Synthetic Datasets 3 with Different Dimensions.}
 \label{results3}  
\end{table}

%\section{Detailed Guidance on Approach Selection}\label{secE}

%Hence, we recommend the use of influence function approach on small datasets as the Gram matrix inversion will not be computationally expensive, while for large datasets, we suggest batching so that the estimation will not be degraded by the dataset dividing.

\end{appendices}

\end{document}